\documentclass{article}

\usepackage{microtype}
\usepackage{graphicx}
\usepackage{booktabs} %

\usepackage{hyperref}

\usepackage[accepted]{icml2023}

\usepackage{amsmath}
\usepackage{amssymb}
\usepackage{mathtools}
\usepackage{amsthm}

\usepackage[capitalize,noabbrev]{cleveref}

\usepackage{caption}
\usepackage{subcaption}
\usepackage{stfloats}
\usepackage{listings}
\allowdisplaybreaks

\theoremstyle{plain}
\newtheorem{theorem}{Theorem}[section]

\newtheorem{lemma}[theorem]{Lemma}

\theoremstyle{definition}

\theoremstyle{remark}

\newtheorem{examplex}[theorem]{Example}
\newenvironment{example}
  {\pushQED{\qed}\examplex}
  {\popQED\endexamplex}

\usepackage[textsize=tiny]{todonotes}

\DeclareMathOperator*{\F}{\mathcal{F}}
\DeclareMathOperator*{\G}{\mathcal{G}}

\DeclareMathOperator{\GP}{GP}

\newcommand{\ke}{k_\mathrm{EPGP}}
\newcommand{\ks}{k_\mathrm{S-EPGP}}

\newcommand{\CC}{\mathbb{C}}

\newcommand{\EE}{\mathbb{E}}
\newcommand{\RR}{\mathbb{R}}
\newcommand{\PP}{\mathbb{P}}

\newcommand{\xx}{\mathbf{x}}
\newcommand{\yy}{\mathbf{y}}
\newcommand{\zz}{\mathbf{z}}
\newcommand{\pphi}{\mathbf{\phi}}

\icmltitlerunning{GP Priors for Systems of Linear PDEs with Constant Coefficients}

\begin{document}

\twocolumn[
\icmltitle{Gaussian Process Priors for Systems of Linear Partial Differential Equations with Constant Coefficients}

\icmlsetsymbol{equal}{*}

\begin{icmlauthorlist}
\icmlauthor{Marc Härkönen}{mpi,fano}
\icmlauthor{Markus Lange-Hegermann}{owl}
\icmlauthor{Bogdan Raiță}{gt}
\end{icmlauthorlist}

\icmlaffiliation{mpi}{Max Planck Institute for Mathematics in the Sciences, Leipzig, Germany}
\icmlaffiliation{fano}{Fano Labs, Hong Kong SAR, China}
\icmlaffiliation{owl}{Institute Industrial IT, OWL University of Applied Sciences and Arts, Lemgo, Germany}
\icmlaffiliation{gt}{Georgetown University, Washington, D.C., USA.
}

\icmlcorrespondingauthor{Marc Härkönen}{marc.harkonen@gmail.com}

\icmlkeywords{Gaussian Processes, Nonlinear algebra, Partial Differential Equations}

\vskip 0.3in
]

\printAffiliationsAndNotice{}  %

\begin{abstract}
  Partial differential equations (PDEs) are important tools to model physical systems and including them into machine learning models is an important way of incorporating physical knowledge. %
  Given any system of linear PDEs with constant coefficients, we propose a family of Gaussian process (GP) priors, which we call EPGP, such that all realizations are exact solutions of this system.
  We apply the Ehrenpreis-Palamodov fundamental principle, which works as a non-linear Fourier transform, to construct GP kernels mirroring standard spectral methods for GPs.
  Our approach can infer probable solutions of linear PDE systems from any data such as noisy measurements, or pointwise defined initial and boundary conditions.
  Constructing EPGP-priors is algorithmic, generally applicable, and comes with a sparse version (S-EPGP) that learns the relevant spectral frequencies and works better for big data sets.
  We demonstrate our approach on three families of systems of PDEs, the heat equation, wave equation, and Maxwell's equations, where we improve upon the state of the art in computation time and precision, in some experiments by several orders of magnitude.
\end{abstract}

\section{Introduction}

Gaussian processes (GPs) \cite{RW} are a major tool in probabilistic machine learning and serve as the default functional prior in Bayesian statistics.
GPs are specified by a mean function and a covariance function.
The covariance function in particular can be constructed flexibly to allow various kinds of priors \cite{TLRB_gp} and learning hyperparameters in GPs allows to interpret data \cite{duvenaud2014automatic,steinruecken2019automatic,BernsTowards2020}.
They serve as stable regression models in applications with few data points and provide calibrated variances of predictions.
In particular, they can serve as simulation models for functions that are costly to evaluate, e.g.\ in Bayesian optimization \cite{hernandez2022designing} or active learning \cite{zimmer2018safe}.
Furthermore, GPs are often the models of choice to encode mathematical information in a prior or if mathematical results should be extracted from a model.
One example is the estimation of derivatives from data by differentiating the covariance function  \cite{swain2016inferring,harrington2016differential}.

These techniques using derivatives have been generalized to construct GPs with realizations in the solution set of specific systems of linear partial differential equations (PDEs) with constant coefficients \cite{MacedoCastro2008,scheuerer2012covariance,Wahlstrom13modelingmagnetic,MagneticFieldGP,jidling2018probabilistic,sarkka2011linear}.
These constructions interpret such a solution set as the image of some latent functions under a linear operator matrix.
Assuming a GP prior for these latent function leads to a GP prior for the solution set of the system of PDEs.
\citet{LinearlyConstrainedGP} pointed out that these constructions of GP priors had striking similarities and suggested an approach for a general construction, after which \citet{LH_AlgorithmicLinearlyConstrainedGaussianProcesses} reinterpreted this approach in terms of Gröbner bases and made it algorithmic.
One limitation was that the method could only work on a subclass of systems of linear PDEs with constant coefficients: the so-called \emph{controllable} (or \emph{parametrizable}) systems.
The restriction to such controllable systems was lifted for systems of ordinary differential equations (ODEs) in \cite{besginow2022constraining}.

In this paper, we develop an algebraic and algorithmic construction of GP priors inside the solution set of \emph{any} given system of (ordinary or partial) linear differential equations with constant coefficients, eliminating previous restrictions to special forms of equations, controllable systems, or ODEs.
Our construction is built upon the classical Ehrenpreis-Palamodov fundamental principle (see Section~\ref{section_EP}) and recent algorithms for the construction of Noetherian multipliers used in this theorem \cite{CCHKL,homs21primary,cidruiz2021primary,chen22primary,manssour21linear}.

\begin{figure*}[t]
  \vskip 0.2in
  \centering
  \begin{subfigure}[b]{0.3\textwidth}
    \centering
    \includegraphics[width=\textwidth]{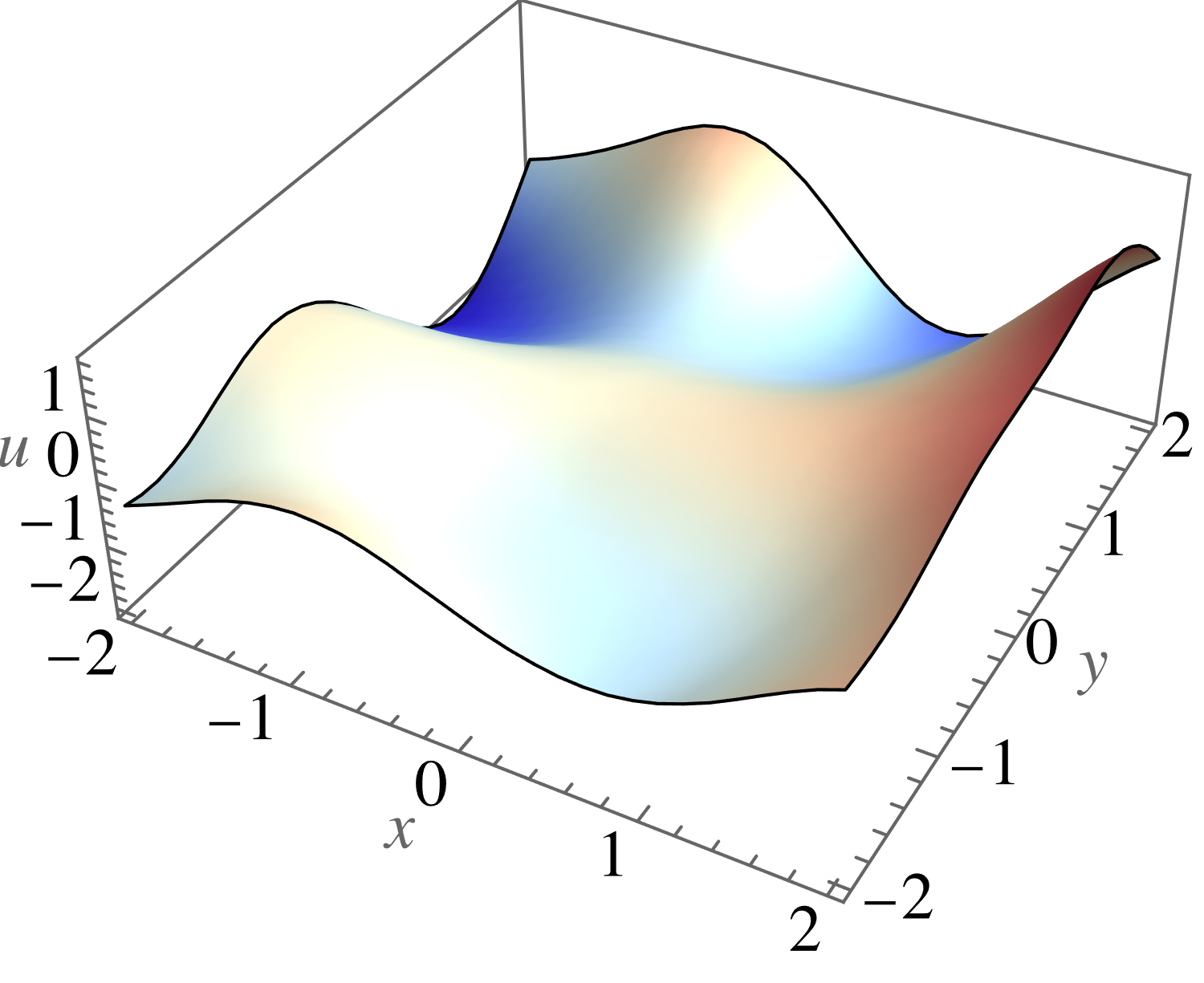}
    \caption{$t = 0$}
  \end{subfigure}
  \hfill
  \begin{subfigure}[b]{0.3\textwidth}
    \centering
    \includegraphics[width=\textwidth]{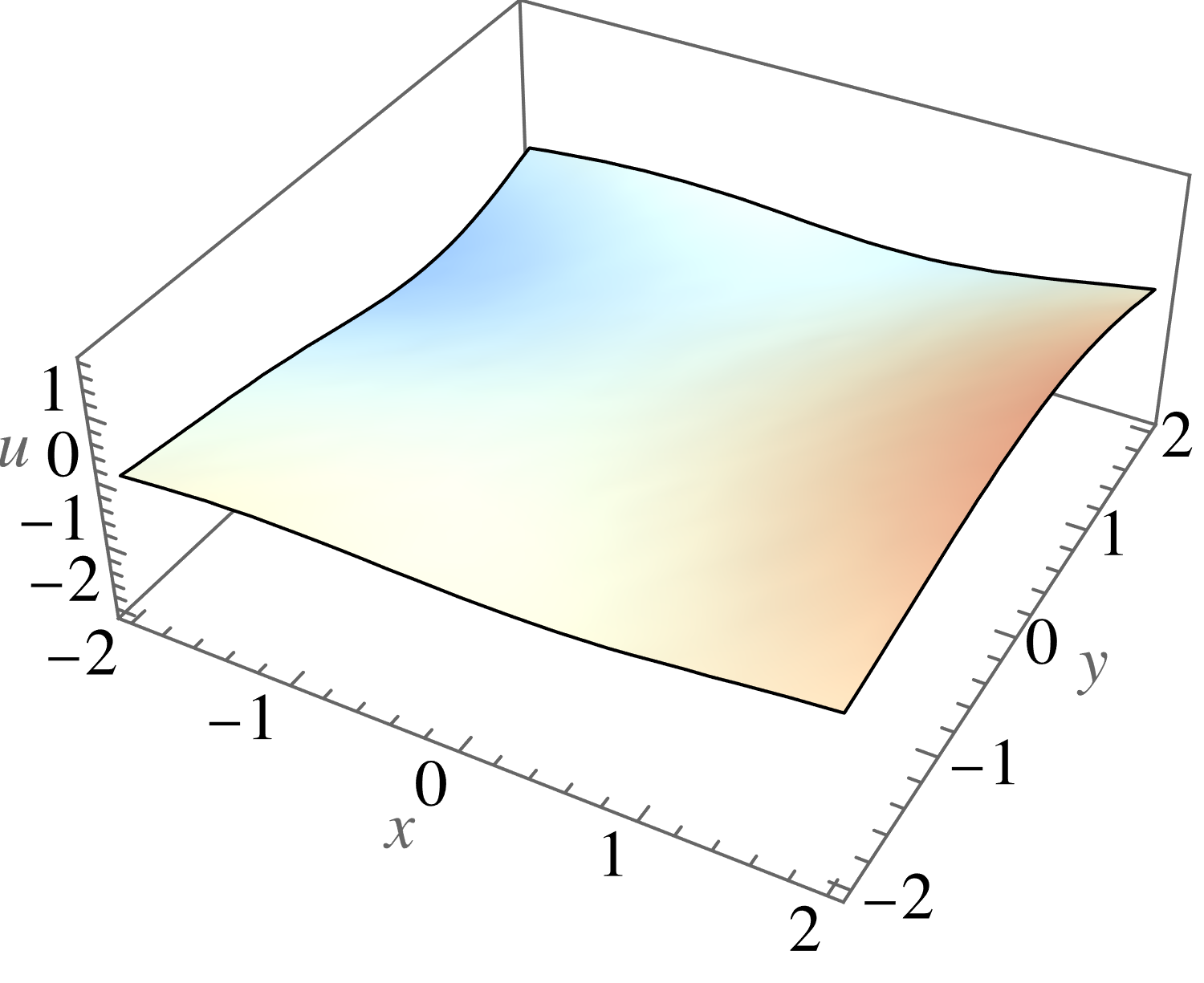}
    \caption{$t = 0.5$}
  \end{subfigure}
  \hfill
  \begin{subfigure}[b]{0.3\textwidth}
    \centering
    \includegraphics[width=\textwidth]{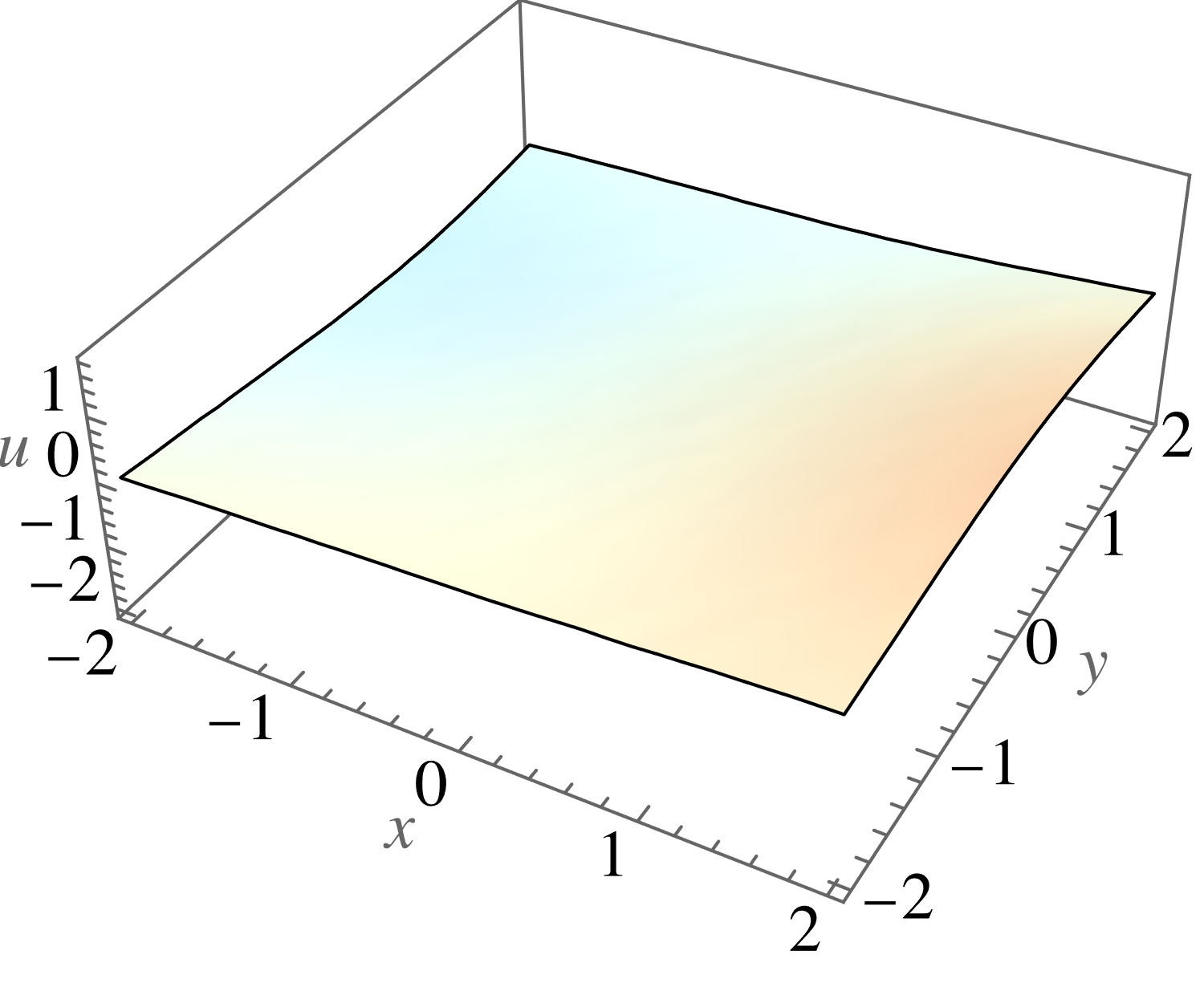}
    \caption{$t = 1$}
  \end{subfigure}
  \caption{Sample of a function $u(x,y,t)$ describing the temperature of a 2-dimensional material over time. The sample is obtained using the EPGP covariance kernel solving the heat equation in 2D. We observe heat dissipating as time progresses.}
  \label{fig:2dheat_sample}
  \vskip -0.2in
\end{figure*}

The major contributions of this paper are as follows:
\begin{enumerate}
    \item We \emph{vastly generalize} previously isolated methods to model systems of ODEs and PDEs from data, such that only the restrictions of linearity and constant coefficients remain and reinterpret the previously existing methods in our framework.
    All previous approaches are special cases: they yield \emph{precisely the same} covariance function as EPGP, in the rare cases where they are applicable.
    \item We do not distinguish between different types of PDEs, e.g. elliptic, hyperbolic, or being of a certain order.
    In particular, we construct GP kernels for \emph{any} linear system of linear PDEs $Af=0$ with constant coefficients.
    \item We demonstrate our approach on various PDE systems, in particular the \emph{homogeneous} Maxwell equations (see \Cref{ssec:maxwell}). 
    \emph{None} of the previously mentioned GP methods is applicable for \emph{any} of the examples studied in Section~\ref{sec:examples}.
    \item We demonstrate high accuracy of our approach, clearly improving upon Physics Informed Neural Network (PINN) methods in several examples (see \Cref{sec:examples}).
\end{enumerate}

Hence, this paper allows the application of machine learning techniques for a vast class of differential equations ubiquitous in physics and numerical analysis.
In particular, we propose a symbolic framework for turning physical knowledge from differential equations into a form usable in machine learning.
The symbolic approach allows us to sample and regress on \emph{exact} solutions of the PDE system, making our methods not merely physics \emph{informed}, but truly physics \emph{constrained}.
Therefore, our GPs result in more precise regression models, since they do not need to use information present in the data or additional collocation points to learn or fit to the differential equation, and instead combine the full information content of the data with differential equations. 

For code and videos corresponding to this paper we refer to the \href{https://arxiv.org/abs/2212.14319}{arXiv}, \href{https://github.com/haerski/epgp}{github}, or \href{https://mathrepo.mis.mpg.de/EPGP/index.html}{mathrepo}.

\section{Gaussian processes (GPs)}

A \emph{Gaussian process (GP)} $g \sim \GP(\mu,k)$ defines a probability distribution on the evaluations of functions $\Omega\to\RR^\ell$, where $\Omega\subseteq\RR^n$, such that function values $g(x_1),\ldots,g(x_m)$
at any points $x_1,\ldots,x_m\in\Omega$ are jointly Gaussian.
A GP $g$ is specified by a \emph{mean function} $\mu:\Omega\to\RR^\ell:x\mapsto \EE[g(x)]$, often a-priori chosen to be zero, and a positive semidefinite%
, smooth \emph{covariance function}
\begin{align*}
  k: \Omega\times\Omega &\longrightarrow \RR^{\ell\times\ell}_{\succeq0} \\
  (x,x') &\longmapsto \EE\left[(g(x)-\mu(x))(g(x')-\mu(x'))^T\right]\mbox{.}
\end{align*}
Then, any finite set of evaluations $[g(x_1), \dotsc, g(x_m)]$ follows the multivariate Gaussian distribution with mean $[\mu(x_1),\dotsc,\mu(x_m)]$ and covariance $\Sigma_{i,j} = k(x_i, x_j)$.
Due to the properties of Gaussian distributions, the posterior is again a GP and can be computed in closed form via linear algebra \cite{RW}.

GPs interplay nicely with linear operators, the foundation of the constructions of \cite{MacedoCastro2008,scheuerer2012covariance,Wahlstrom13modelingmagnetic,MagneticFieldGP,jidling2018probabilistic,sarkka2011linear,LinearlyConstrainedGP,LH_AlgorithmicLinearlyConstrainedGaussianProcesses,besginow2022constraining,LH_AlgorithmicLinearlyConstrainedGaussianProcessesBoundaryConditions,langehegermann2022boundary}:

\begin{lemma}%
\label{lemma_pushforward_gaussian}
	Let $g \sim \GP(\mu(x),k(x,x'))$ with realizations in some function space $\F^\ell$, $\F=C^\infty(\Omega)$, and $B:\F^\ell\to\F^{\ell''}$ a linear, continuous operator.
	Then, the \emph{pushforward} $B_*g$ of $g$ under $B$ is a GP with
	\begin{align}\label{eq_tansformation}
	B_*g \sim \GP(B\mu(x),Bk(x,x')(B')^T)\mbox{ ,}
	\end{align}
	where $B'$ denotes the operation of $B$ on functions with argument $x'$.
\end{lemma}
The proof we give  in \Cref{appendix_proof_lemma_pushforward_gaussian} in fact works for linear continuous operators $B\colon F^\ell\rightarrow \G^{\ell''}$ for spaces $\F,\,\G$ which embed continuously in the space $C(\Omega)$ of continuous functions. The analogous result holds in $L^p$ spaces, which is out of scope of the current paper.
For complex valued GPs, one replaces the transpose by the Hermitian transpose.

\section{The Ehrenpreis-Palamodov fundamental principle}\label{section_EP}

Consider the familiar case of a linear ODE with constant coefficients, e.g.\ $f'''(x)-3f'(x)+2f(x)=0$.
The solution space of the ODE is determined by its \emph{characteristic polynomial} $z^3-3z+2$ via its roots and their multiplicities.
In this case $z^3-3z+2$ factors into $(z-1)^2\cdot(z-(-2))$, so all solutions are linear combinations of the three functions $e^{1x}$, $x\cdot e^{1x}$ and $e^{-2x}$.
We call functions of the form $D(x)\cdot e^{zx}$ \emph{exponential-polynomial} functions whenever $D(x)$ is a polynomial and $z$ is a constant.
This idea generalizes to systems of ODEs and PDEs: instead of taking linear combinations of exponential-polynomial functions over the finitely many zeros of the characteristic polynomial, one takes a weighted integral of exponential-polynomial functions over a (potentially multi-dimensional) characteristic variety\footnote{A variety is defined as the zero set of a system of polynomials.}.
This generalization is formalized in the Ehrenpreis-Palamodov fundamental principle, \Cref{thm:ehrenpreis-palamodov}.

More formally, let $\Omega$ be a compact, convex subset of $\RR^n$.
Consider systems of $\ell$ equations with smooth functions $f \colon \Omega \to \CC^{\ell''}$ as potential solutions. %
We encode such a system of PDEs as an $(\ell \times {\ell''})$ matrix $A$ with entries in the polynomial ring $R = \CC[\partial_1,\dotsc,\partial_n]$ in $n$ variables.
Here the symbol $\partial_i$ denotes the operator $\frac{\partial}{\partial x_i}$ and a monomial $\partial^\alpha = \partial_1^{\alpha_1}\dotsb\partial_n^{\alpha_n}$ denotes the operator $\frac{\partial^{|\alpha|}}{\partial x_1^{\alpha_1}\dotsb\partial x_n^{\alpha_n}}$.
For example, for $\ell = 3, {\ell''} = 2, n = 2$, the PDE system
\begin{align*}
  Af
  =
  \begin{bmatrix}\partial_1 & -\partial_1^2\partial_2 \\ \partial_2 & 1 \\ 0 & -\partial_1 + 3 \partial_2 \end{bmatrix} f
  = 0
\end{align*}
translates to a system of 3 homogeneous equations
\begin{align*}
 \frac{\partial f_1}{\partial x_1} - \frac{\partial^3 f_2}{\partial x_1^2 \partial x_2} = \frac{\partial f_1}{\partial x_2} + f_2  = - \frac{\partial f_2}{\partial x_1} + 3 \frac{\partial f_2}{\partial x_2} = 0.
\end{align*}
Its solutions are vector valued functions $f(x_1,x_2) = (f_1(x_1,x_2), f_2(x_1,x_2))^T$.
Another example is the 2 dimensional heat equation, where $A$ is the $1 \times 1$ matrix $A = [\partial_x^2 + \partial_y^2 - \partial_t]$.
Its solutions are scalar functions $u(x,y,t)$, such as the one displayed in \Cref{fig:2dheat_sample}.

The famed Ehrenpreis-Palamodov fundamental principle asserts that all solutions to the PDEs represented by $A$ can be written as suitable \emph{integrals} of exponential-polynomial solutions, each of which corresponds to roots and multiplicities of the polynomial module generated by rows of $A$.

\begin{theorem}{\citep{EHRENPREIS,PALAMODOV,HORMANDER,BJORK}}\label{thm:ehrenpreis-palamodov}
  Let $A \in R^{\ell \times {\ell''}}$ and let $\Omega \subseteq \RR^n$ be a convex, compact set.
  There exist algebraic varieties $\{V_1,\dotsc,V_s\}$ and ${\ell''}$-tuples of polynomials $\{D_{i,1}(\xx, \zz), \dotsc, D_{i,m_i}(\xx,\zz)\}_{i=1,\dotsc,s}$
  such that any smooth solution $f \colon \Omega \to \RR^{\ell''}$ to the equation $Af = 0$ can be written as
  \begin{align}\label{eq:ehrenpreis-palamodov}
    f(\xx) &= \sum_{i=1}^s \sum_{j=1}^{m_i} \int_{V_i} D_{i,j}(\xx, \zz) e^{\langle \xx, \zz \rangle}%
    \,d\mu_{i,j}(\zz)%
  \end{align}
  for a suitable choice of measures $\mu_{i,j}$.%
\end{theorem}

Following the terminology in \cite{homs21primary,cidruiz2021primary}, we call the polynomials $D_{i,j}(\xx,\zz)$ \emph{Noetherian multipliers}.
The Noetherian multipliers $D_{i,j}$ and varieties $V_i$ appearing in \Cref{thm:ehrenpreis-palamodov} can be computed algebraically; they are the higher-dimensional analogue of the roots and multiplicities of the characteristic polynomial.
An algorithm for computing $D_{i,j}$ and $V_i$ is implemented under the command \texttt{solvePDE} in the Macaulay2 \cite{grayson2002macaulay2} package \texttt{NotherianOperators} \cite{CCHKL}.
 A modern, algebraic and algorithmic treatment of linear PDEs with constant coefficients can be found in \cite{homs21primary,cidruiz2021primary,chen22primary,manssour21linear}.
 We refer the interested reader to \Cref{app:supporting_function} for questions regarding convergence of the integrals in \eqref{eq:ehrenpreis-palamodov}.

\section{Gaussian Process Priors from the Ehrenpreis-Palamodov Theorem}

We now construct GPs whose samples %
solve a system of linear PDEs $Af=0$, using the Ehrenpreis-Palamodov fundamental principle, \Cref{thm:ehrenpreis-palamodov}, as a blueprint.
We set the mean function to zero, so our task, by \Cref{lemma_pushforward_gaussian}, will be to find a covariance function that satisfies the PDEs in both the $\xx$ and $\xx'$ arguments.
The varieties $V_i$ and polynomials $D_{i,j}$ in \cref{eq:ehrenpreis-palamodov} can be computed algorithmically \cite{CCHKL,homs21primary,cidruiz2021primary,chen22primary,manssour21linear}, so what remains is to choose the measures $\mu_{i,j}$, each supported on the variety $V_i$.

We propose two approaches for choosing the measures.
In the first one, coined Ehrenpreis-Palamodov Gaussian Process (EPGP), the $\mu_{i,j}$ are chosen to be Gaussian measures supported on the variety, with optional trainable length scale and shift parameters. 
This resembles the construction by \citet{wilson2013gaussian}, but applied to Ehrenpreis-Palamodov integrals as opposed to Fourier transforms.
Our second approach, Sparse EPGP (S-EPGP), chooses $\mu_{i,j}$ to be linear combinations of Dirac delta measures, whose locations and weights are learned.
See \cite{lazaro20sparse} for a similar approach applied to Fourier transforms.

Before describing our covariance functions, we discuss the question of how to integrate over an algebraic variety.
In certain cases our variety has a polynomial parametrization, in which case the integral can easily be computed by substituting the parametrization in.
For example if $V$ is the variety corresponding to the parabola $y = x^2$, we can rewrite an integral $\int_V f(x,y) \, d\mu(x,y)$ over $V$ as $\int_\CC f(x,x^2) \, d\mu'(x)$.

However, most algebraic varieties $V$ do not have a parametrization.
In these cases we construct a parametrization implicitly by solving equations.
If for example the variety $V$ is the set of points $(x,y,z)$ where $x^3-y^2+z^2 = 0$, we could solve for $z$ to get $z = \pm \sqrt{y^2 - x^3}$.
Thus, an integral of the form $\int_V f(x,y,z) \, d\mu(x,y,z)$ can be rewritten as a sum $\int_{\CC^2} f(x,y,\sqrt{y^2-x^3}) \, d\mu_1(x,y) + \int_{\CC^2} f(x,y,-\sqrt{y^2-x^3}) \, d\mu_2(x,y)$ of integrals over $\CC^2$.
This construction works for arbitrary varieties $V \subseteq \CC^n$ and relies on results in algebraic dimension theory.
We now cite the main results and refer to e.g.~ the textbook by \citet[Sec.~13.1]{eisenbud1995commutative} for a comprehensive treatment.

Suppose we denote the coordinates of $\CC^n$ by $z_1,\dotsc,z_n$.
If $V$ has dimension $d$, there is a set of $d$ \emph{independent variables}, say $\zz' = (z_1,\dotsc,z_d)$ after reordering, on which the remaining variables $\zz'' = (z_{d+1},\dotsc,z_{n})$ depend algebraically.
Thus for each choice of $\zz' \in\CC^d$, there is a finite number of $\zz'' \in \CC^{n-d}$ such that $\zz = (\zz', \zz'') \in V$.
We denote this set by $\mathcal{S}_{\zz'} := \{\zz \in V \colon (z_1,\dotsc,z_d) = \zz' \}$.
Using this notation, an integral $\int_V f(\zz) \, d\mu(\zz)$ over $V$ can now be rewritten as an integral $ \int_{\CC^d} \sum_{\zz \in \mathcal{S}_{\zz'}} f(\zz) \, d\mu(\zz')$ over the much easier to handle affine space $\CC^d$, at the cost of changing the measure and splitting our integral into several pieces.

\subsection{Ehrenpreis-Palamodov Gaussian Processes (EPGP)} \label{sec:epgp}
Let $Af=0$ be a system of PDEs whose solutions are, by Ehrenpreis-Palamodov, of the form $\phi(\xx) = \sum_j \int_V D_j(\xx,\zz) e^{\langle \xx, \zz \rangle} \, d\mu_j(\xx)$.
We define the EPGP kernel $\ke(\xx,\xx')$ by combining the Ehrenpreis-Palamodov representation in both inputs $\xx, \xx'$, the above implicit parametrization of the integrals, and a Gaussian measure on the frequency space of the $\zz$.
We construct one covariance kernel for each summand in $\phi(\xx)$ and sum them to get the EPGP kernel $\ke$:
\begin{align}
  \Psi(\xx,\zz') &:= \sum_j \sum_{\zz \in \mathcal{S}_{\zz'}} D_j(\xx,\zz) e^{\langle \xx, \zz \rangle} \nonumber \\
    \ke(\xx,\xx') &:= \label{eq:epgp}\\
                  & \hspace{-2em} \int_{\zz' \in \sqrt{-1}\RR^d} \Psi(\xx,\zz') \Psi(\xx',\zz')^H e^{-\frac{\|\zz'\|^2}{2}} \, d\mathcal{L}(\zz'). \nonumber
\end{align}
Here the superscript $H$ denotes the Hermitian transpose and $\mathcal{L}$ is the usual Lebesgue measure.
We note that the integral may not converge everywhere, but we can introduce a shifting term to enforce convergence in any compact set $\Omega$. See \Cref{app:supporting_function} for details.
It is straightforward to check that $\ke(\xx,\xx')$ satisfies the PDEs in $A$ and the Hermitian transpose ensures that $\ke$ is positive semidefinite.
A strictly real valued GP is obtained by taking the real part of $\ke$.

In \eqref{eq:epgp}, we replaced the integral over the complex space $\zz' \in \CC^d$ by an integral over purely imaginary vectors $\zz' \in \sqrt{-1} \RR^d$.
This leads to more stationary kernels and we further motivate this choice and the choice of the Gaussian measure, in three examples:

\begin{example}[No PDE]\label{ex:no_pde}
If we impose no PDE constraints, we have $A = 0$, one variety $V = \CC^n$ and one Noetherian multiplier $D(\xx,\zz) = 1$. So equation \eqref{eq:epgp} becomes
\begin{align*}
  \ke(\xx,\xx') &= \int_{\zz \in \RR} e^{\sqrt{-1}\langle \xx-\xx', \zz \rangle} e^{-\frac{\|\zz\|^2}{2}} \, d\mathcal{L}(\zz)  \\
                &= \left(\sqrt{2\pi}\right)^n e^{-\frac{\|\xx-\xx'\|^2}{2}}
\end{align*}
Thus, without PDEs, the EPGP kernel is the squared-exponential kernel, up to a constant scaling factor.

The discussion in this example extends to any system of PDEs $A$ whose characteristic variety $V$ is an affine subspace of $\CC^n$.
For details, refer to \Cref{sec:app_affine}, cf.\ also \cite{LH_AlgorithmicLinearlyConstrainedGaussianProcessesBoundaryConditions}.
\end{example}

\begin{example}[Heat equation]\label{ex:heat}
Let $A(\partial_x, \partial_t) = \partial_x^2 - \partial_t$ be the one-dimensional heat equation. If we let $z_1, z_2$ correspond to $\partial_x, \partial_t$ respectively, the variety $V$ is given by $z_1^2 = z_2$ and the sole Noetherian multiplier is $D=1$.
The EPGP kernel is defined when $t + t' > -\frac{1}{2}$, in which case we have
\begin{align*}
  \ke(x,t;x',t') &= \int_{z \in \RR} e^{\sqrt{-1}(x-x')} e^{-z^2(t+t')} e^{-\frac{z^2}{2}} \, d\mathcal{L}(z) \\
                 &= \sqrt{2\pi} \frac{e^{-\frac{(x-x')^2}{2(1+2(t+t'))}}}{\sqrt{1+2(t+t')}}
\end{align*}
Here, integrating over $\sqrt{-1}\RR$ as opposed to $\CC$ removes unphysical solutions to the heat equation, such as $\phi(x,t) = e^{x+t}$ where heat increases exponentially with time.

This covariance function is the squared exponential covariance w.r.t.\ the space dimension at each fixed pair of times $(t,t')$.
With increasing time, the scaling of the covariance shrinks resp.\ the length scales increase.
We interpret this as heat going back to the mean value resp.\ being more smoothly distributed over time.
\end{example}

\begin{example}[Wave equation]\label{ex:wave}
Let $A(\partial_x,\partial_t) = \partial_x^2 - \partial_t^2$ be the 1-dimensional wave equation.
The variety here is the set $z_1^2 - z_2^2 = 0$, which is the union of the lines $z_1 = z_2$ and $-z_1 = z_2$.
Thus we have $V_j = V(z_1+(-1)^jz_2)$ and $D_j = 1$ for $j=1,2$.
The EPGP kernel is equal to
\begin{gather*}
\begin{split}
  \ke(x,t;x',t') = \int_{z \in \RR} \left( e^{\sqrt{-1}z(x-t)} + e^{\sqrt{-1}z(x+t)} \right) \cdot \\ \left( e^{-\sqrt{-1}z(x'-t')} + e^{\sqrt{-1}z(x'+t')} \right) e^{-\frac{z^2}{2}} \, d\mathcal{L}(z)
\end{split} \\
\begin{split}
  = \sqrt{2\pi} \bigl( e^{-\frac{((x-t)-(x'-t'))^2}{2}} + e^{-\frac{((x-t)-(x'+t'))^2}{2}} + \\ e^{-\frac{((x+t)-(x'-t'))^2}{2}} + e^{-\frac{((x+t)-(x'+t'))^2}{2}}\bigr)
\end{split}
\end{gather*}
Here our choice of restricting to integrals over strictly imaginary numbers $\sqrt{-1}\RR$ gets rid of non-stable solutions to the wave equations, such as $e^{x+t}$.

While $\ke$ has four summands, we can also consider kernels with fewer summands.
E.g.\
\begin{align*}
  k_2(x,t;x',t') = e^{-\frac{((x+t)-(x'+t'))^2}{2}} + e^{-\frac{((x-t)-(x'-t'))^2}{2}},
\end{align*}
yields the covariance kernel of the GP $\phi_1(x+t) + \phi_2(x-t)$, where $\phi_1,\phi_2 \sim GP(0, e^{-\frac{(x-x')^2}{2}})$.
This is kernel in fact covers all smooth solutions to the 1-D wave equation, as d'Alembert discovered in 1747 \cite{d1747recherches} that all solutions are superpositions of waves travelling in opposite directions.
\end{example}

Here we have tacitly assumed that the Ehrenpreis-Palamodov integral requires only one variety $V$, i.e.\ $s = 1$.
In the general case $s > 1$, we repeat the above construction for each variety $V_i$ and finally sum the resulting $s$ kernels.

We note that we may also parametrize the Gaussian measure imposed in \cref{eq:epgp}, for example with mean and scale parameters.
By replacing $\exp(-\frac{\|\zz'\|^2}{2})$ by e.g. $\exp(-\sum_{i=1}^d \frac{(z_i - \mu_i)^2}{2\sigma_i^2})$, we obtain a family of EPGP kernels, which we can train on given data to find the parameters $\mu_i, \sigma^2_i$ maximizing the log-marginal likelihood.
In the case of no PDE constraints, we recover exactly the covariance kernels proposed in \cite{wilson2013gaussian}.
In \Cref{fig:face_frames} of \Cref{ssec:heat_eqn} we investigate the effect of a scale parameter in the posterior distribution of a solution to the 2 dimensional heat equation.

\subsection{Sparse Ehrenpreis-Palamodov Gaussian Processes (S-EPGP)}\label{ssec:S-EPGP}
Instead of imposing the measure $e^{-\|\zz'\|^2/2} \,d\mathcal{L}(\zz')$ in our kernel, we outline a computationally efficient method for 
estimating the integral by a weighted sum.
The kernels described in this section resemble the ones in \cite{lazaro20sparse}, but using representations of PDE solutions via the Ehrenpreis-Palamodov fundamental principle as opposed to the Fourier transform.
For notational simplicity we assume that the $D_j(\xx,\zz)$ are scalar valued and there is only one variety in \cref{eq:ehrenpreis-palamodov}, i.e.\ $s=1$; the extension of our analysis to the general case is straight forward and a concrete example of S-EPGP applied to Maxwell's equations can be seen in \Cref{ssec:maxwell}.

The idea is to choose the measures $\mu_{i,j}$ in \cref{eq:ehrenpreis-palamodov} as linear combinations of Dirac ``delta functions''.
Ideally we would define a GP prior with realizations of the form
\begin{align*}
  f(\xx) = \sum_{j=1}^m \sum_{i=1}^r w_{i,j} D_j(\xx, \zz_{i,j}) e^{\langle \xx, \zz_{i,j} \rangle},
\end{align*}
where all $z_{i,j}\in V$.
This is precisely the Ehrenpreis-Palamodov representation of solutions, as in equation~\ref{eq:ehrenpreis-palamodov}, with $r$ Dirac delta measures for each integral.
Given training data, we would then choose $\zz_{i,j} \in V$ as to maximize the log marginal likelihood.
Unfortunately, the requirement for $\zz_{i,j}$ to lie on an algebraic variety makes it challenging to directly use a gradient descent based optimization method.

Instead, we use the implicit parametrization trick from the beginning of this section %
and are looking at a GP with realizations of the form
\begin{align*}
  f(\xx) &= \sum_{j=1}^m \sum_{i=1}^r w_{i,j} \frac{1}{|S_{\zz'_{i,j}}|} \left( \sum_{\zz \in S_{\zz'_{i,j}}} D_j(\xx, \zz) e^{\langle \xx, \zz \rangle} \right) \\
         &=: \mathbf{w}^T \pphi(\xx),
\end{align*}
where now $\zz'_{i,j} \in \CC^d$, $d = \dim V$, and $\mathbf{w}, \pphi(x)$ are both vectors of length $mr$.
For the same reasons as in the previous section we may also choose $\zz'_{i,j} \in \sqrt{-1} \RR^d$.
To turn $f(\xx)$ into a GP, set $w_{i,j} \sim \mathcal{N}\left(0, \frac{1}{mr} \Sigma\right)$, where $\Sigma$ is a diagonal matrix with positive entries $\sigma_{i}^2$ for $i = 1,\dotsc, mr$.
We then get a covariance function of the form
\begin{align}\label{eq:sepgp-kernel}
  \ks(\xx,\xx') = \frac{1}{mr} \pphi(\xx)^H \Sigma \pphi(\xx')
\end{align}
where $\pphi(\xx)^H$ denotes the conjugate transpose of $\pphi(\xx)$.
Refer to \Cref{appendix_SEPGP} for details regarding the S-EPGP objective function and inference.
An example implementation in PyTorch can be found in \Cref{app:code}.

\subsection{Summary}
Below, we summarize concrete steps required to constuct (S-)EPGP kernels.
We emphasize every step can be implemented algorithmically, given a system of PDEs as input.
\begin{enumerate}
  \item Use the Macaulay2 command \texttt{solvePDE} to compute the varieties $V_i$ and Noetherian multipliers.
  \item Find nice parametrizations $\mathbb{C}^d \to V_i$ for each variety. If such a parametrization does not exist, use a combination of a random linear change of coordinates and a univariate polynomial solver. %
  \item Compute the resulting integral \eqref{eq:epgp}. If a closed form solution exists, use it as-is to obtain an EPGP kernel. If not, use S-EPGP \eqref{eq:sepgp-kernel}.
\end{enumerate}
Note that EPGP can be Monte-Carlo approximated using the S-EPGP kernel with frozen, randomly selected $z_{i,j}$ parameters.

\section{Comparison to the Literature}

Physics informed methods are a central research topic in machine learning.
The PINN approach adds additional loss terms for a deep neural network from the differential equations at collocation points, sometimes combined with feature engineering, specific network structures, usage of symmetries and similar techniques, see e.g.\ \cite{milligen1995neural, lagaris2000neural, raissi2019physics, cuomo2022scientific, drygala2022generative}.
Such techniques also include GPs as a tool, e.g.\ 
\cite{zhang2022pagp} uses GPs to estimate solutions of a single PDE where the derivatives of the GPs are used in the loss function and \cite{chen2022apik} uses GPs to model a single function constrained by a single linear PDE.
Another recent approach uses GPs and collocation points to solve linear PDE systems with constant coefficients \cite{pfortner2022physics}.

There are several other deep learning approaches to systems of PDEs.
As an example, weak adversarial networks \cite{zang2020weak} strive for a Nash equilibrium between a neural network that minimizes a weak formulation for a PDE and a second neural network modeling the test function in this weak formulation.
Alternatively, when given a variational formulation of a PDE, where the solution of the PDEs minimizes an integral, the deep Ritz method \cite{yu2018deep} approximates solutions of PDEs via a neural network such that a discrete approximation of the integral is minimized.
Ordinary differential equations have been used to construct deep neural networks \cite{chen2018neural}, which has in turn being used to learn differential equations \cite{saemundsson2020variational}.
See also the review \cite{tanyu2022deep} on similar deep learning methods.

The approaches in \cite{MacedoCastro2008,scheuerer2012covariance,Wahlstrom13modelingmagnetic,MagneticFieldGP,jidling2018probabilistic,sarkka2011linear,LinearlyConstrainedGP,LH_AlgorithmicLinearlyConstrainedGaussianProcesses,dong1989kriging,van2001kriging,albert2019gaussian} construct GPs for controllable systems of linear PDEs with constant coefficients using parametrizations and \cref{lemma_pushforward_gaussian}.
In the language of our paper, the controllable systems are the systems with characteristic variety equal to the full space of frequencies, see Appendix~\ref{sec:app_affine}.
In particular, the approaches in all of the above papers are special cases of our EPGPs.
When we would apply EPGPs to the differential equations treated in these papers, we would get precisely the same results.
However, none of these approaches can treat the three examples that we demonstrate in Section~\ref{sec:examples}, as these examples all have proper characteristic varieties.
\citet{besginow2022constraining} construct priors for all systems of linear ODEs with constant coefficients by splitting apart the embedded components in the characteristic variety and model them via linear regression, whereas the controllable components are again parametrized.
Again, this approach is a special case of EPGPs.

Several paper deal with special cases of controllable systems.
The papers \cite{pmlr-v5-alvarez09a,hartikainen2012sequential,alvarez2013linear,reece2014efficient,alvarado2014latent,ghosh2015modeling,raissi2017machine,camps2018physics,sarkka2018gaussian,nayek2019gaussian,pang2019neural,rogers2020application,gahungu2022adjoint} constructs priors for linear ODE or PDE systems with forcing terms, which are also controllable.
Notably, \cite{ward2020black} used these methods in the context of linearization.
Furthermore, \citet{ranftl2022connection} uses GPs to construct neural networks which only allow approximate solutions to given PDEs as trained functions. 

\begin{figure*}[t]
  \vskip 0.2in
  \begin{center}
    \includegraphics[width=0.5\linewidth]{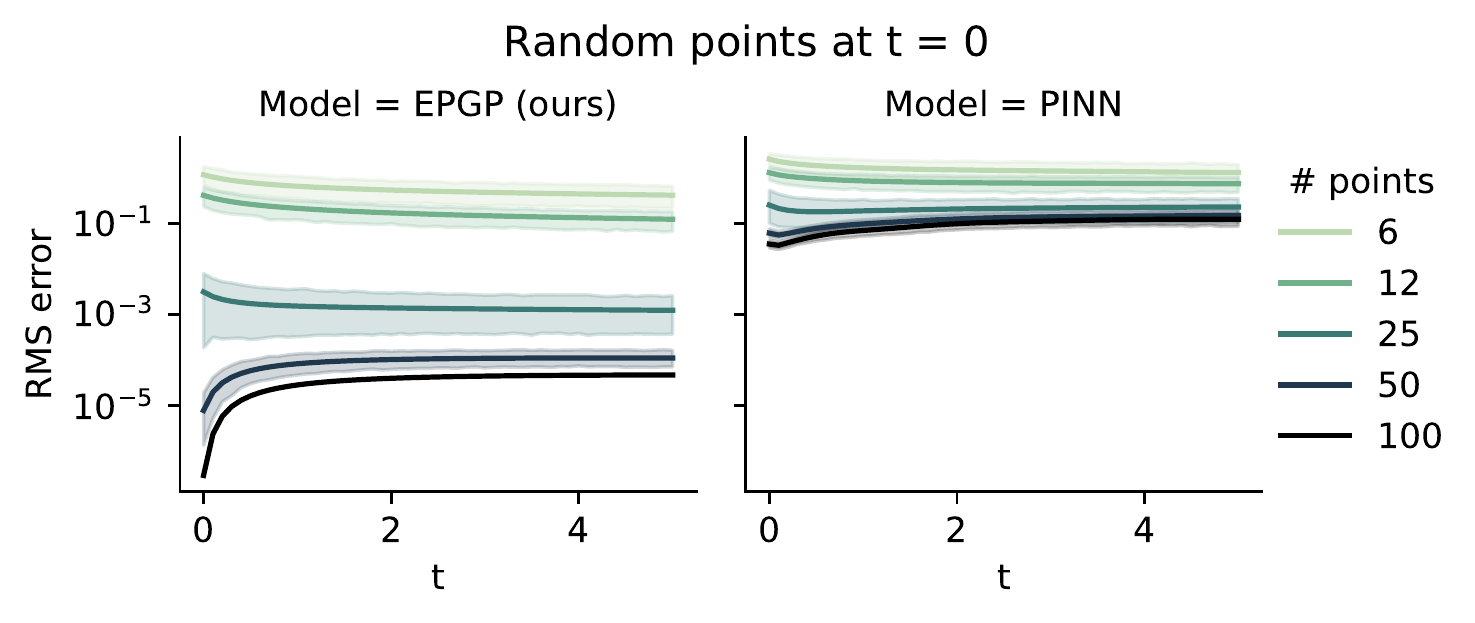}\hfill
    \includegraphics[width=0.5\linewidth]{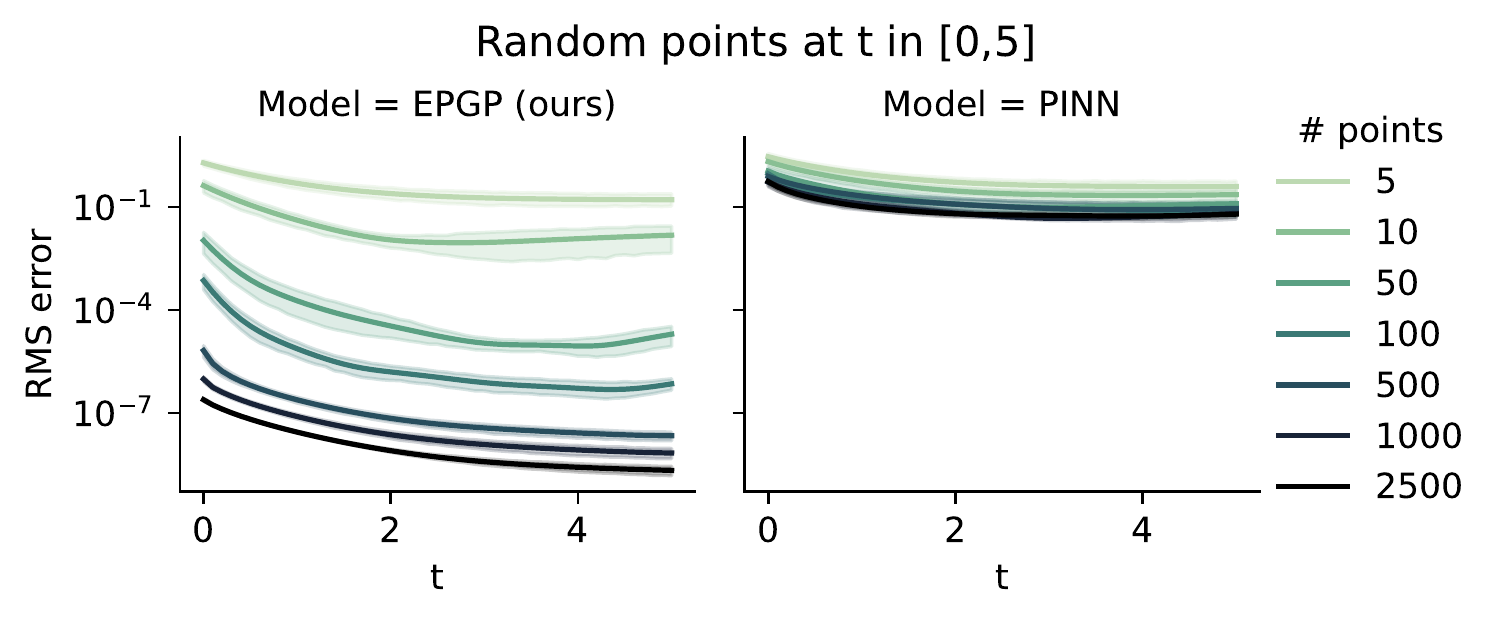}
    \caption{%
      Comparison of the error between EPGP (ours) and PINN for learning a solution to a 1D heat equation.
      On the left, training data is constrained to $t = 0$, which tests PDE solving capabilities based on initial data.
      On the right, training data is spread over the entire interval $t\in[0,5]$ to test interpolation performance.
      EPGP yields considerably better results over a wide range of the amount of training data.
      The error regions stem from training the model on 10 different instances.
    }
    \label{fig:heat1D}
  \end{center}
  \vskip -0.2in
\end{figure*}

\begin{figure*}[t]
  \vskip 0.2in
  \centering
  \begin{subfigure}[b]{0.18\textwidth}
    \centering
    \includegraphics[width=\textwidth]{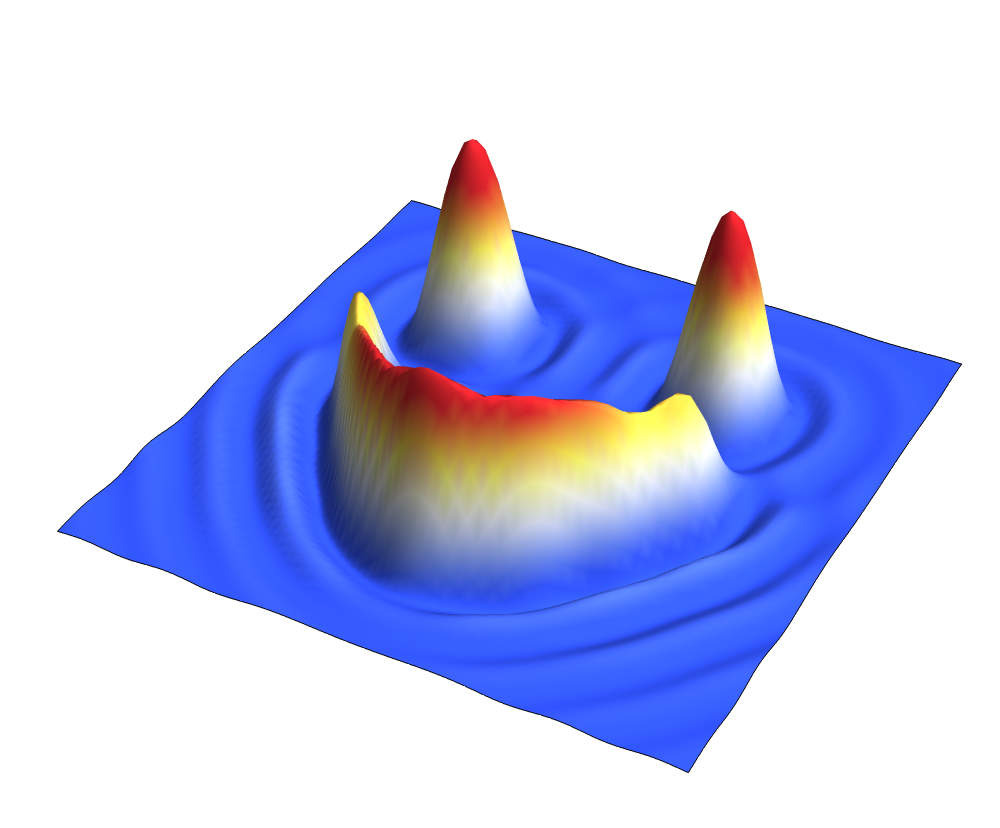}
    \caption{\tiny $t=0.000, \sigma^2 =2$}
  \end{subfigure}
  \hfill
  \begin{subfigure}[b]{0.18\textwidth}
    \centering
    \includegraphics[width=\textwidth]{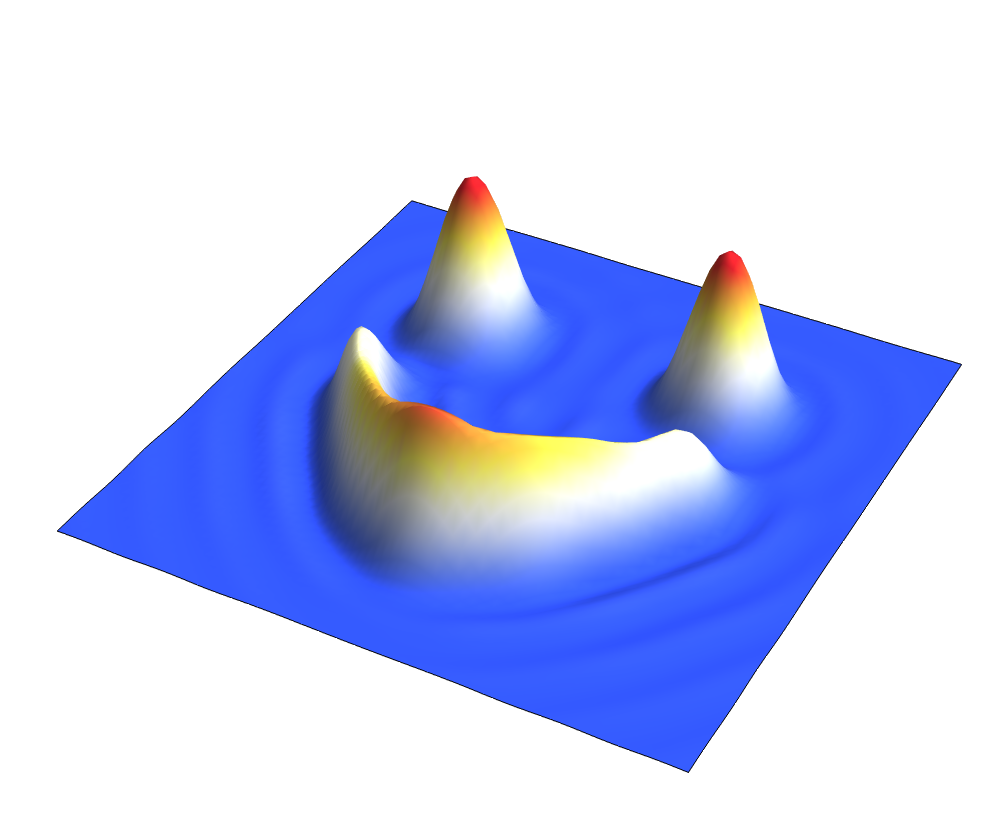}
    \caption{\tiny $t=0.015, \sigma^2 =2$}
  \end{subfigure}
  \hfill
  \begin{subfigure}[b]{0.18\textwidth}
    \centering
    \includegraphics[width=\textwidth]{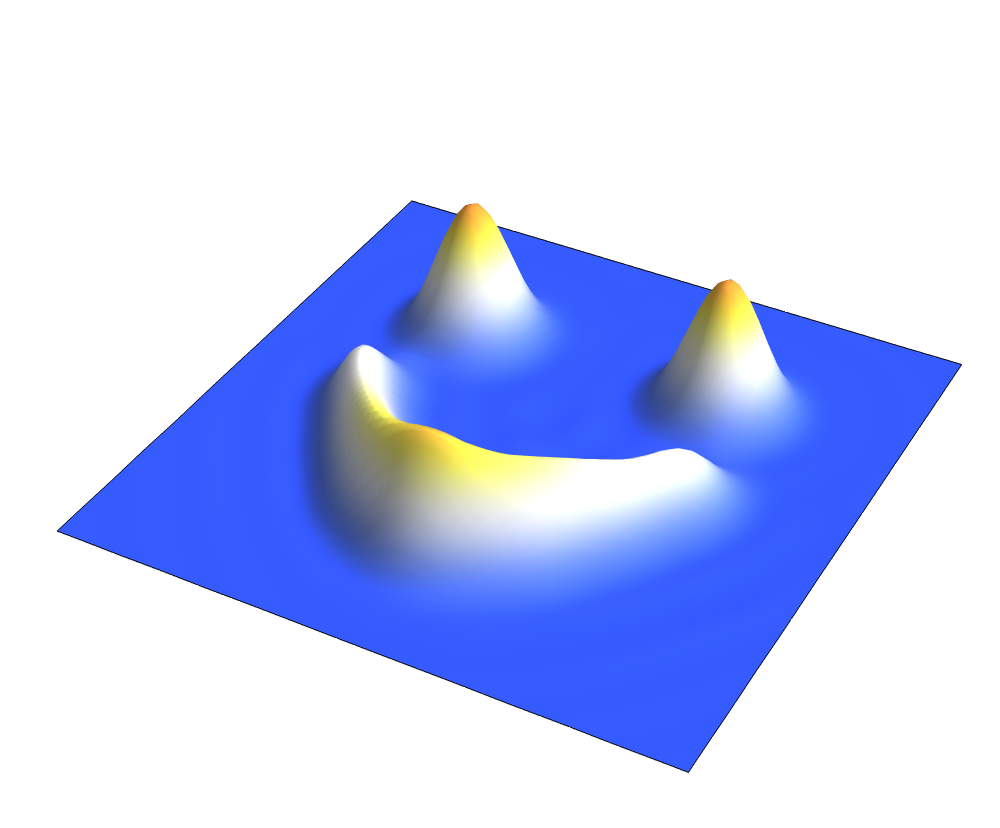}
    \caption{\tiny $t=0.03,\sigma^2 =2$}
  \end{subfigure}
  \hfill
  \begin{subfigure}[b]{0.18\textwidth}
    \centering
    \includegraphics[width=\textwidth]{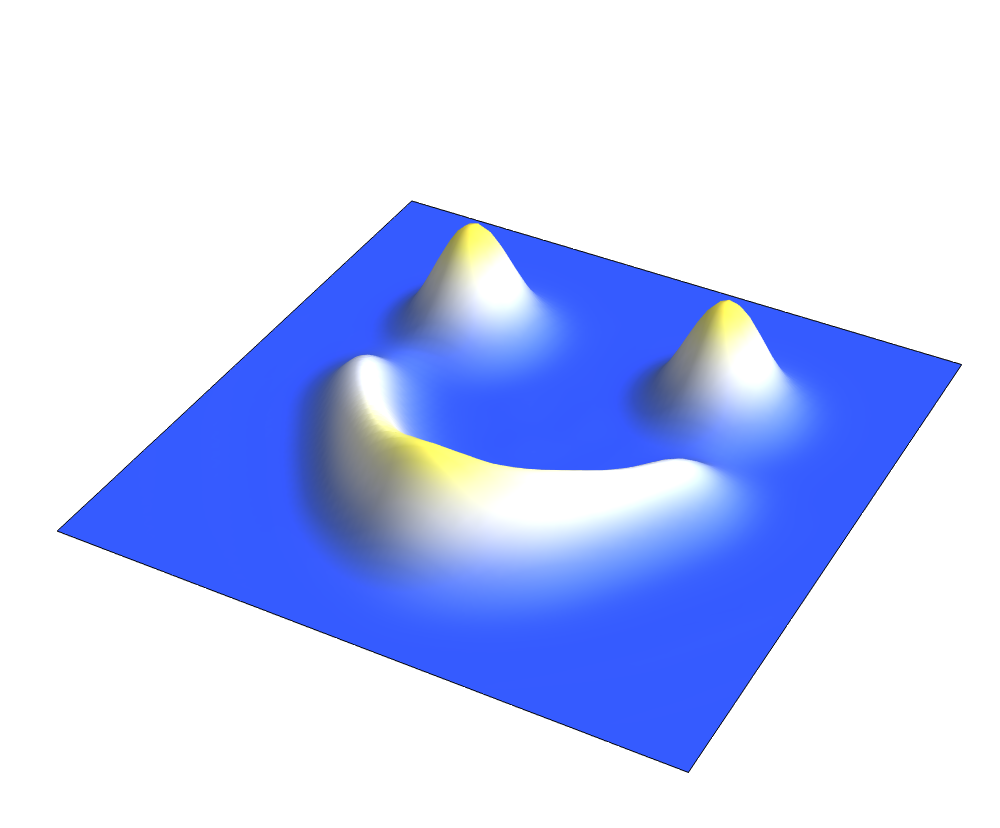}
    \caption{\tiny $t=0.045, \sigma^2 =2$}
  \end{subfigure}
  \hfill
  \begin{subfigure}[b]{0.18\textwidth}
    \centering
    \includegraphics[width=\textwidth]{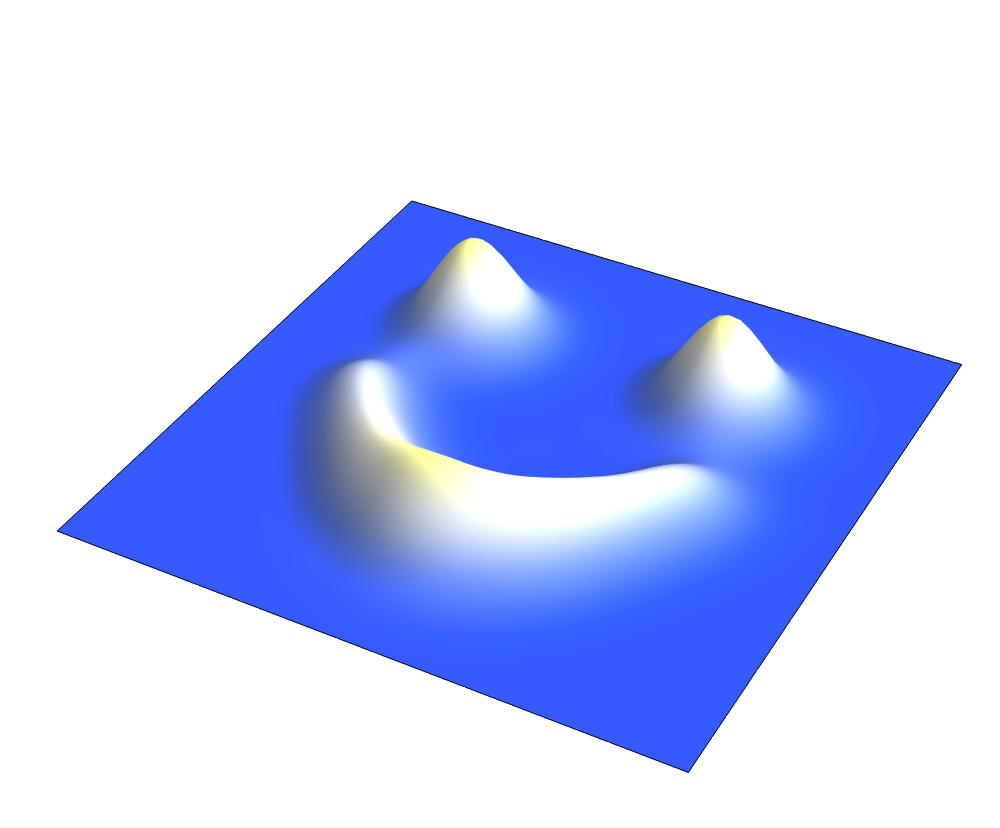}
    \caption{\tiny $t=0.06, \sigma^2 =2$}
  \end{subfigure}
  \\
  \begin{subfigure}[b]{0.18\textwidth}
    \centering
    \includegraphics[width=\textwidth]{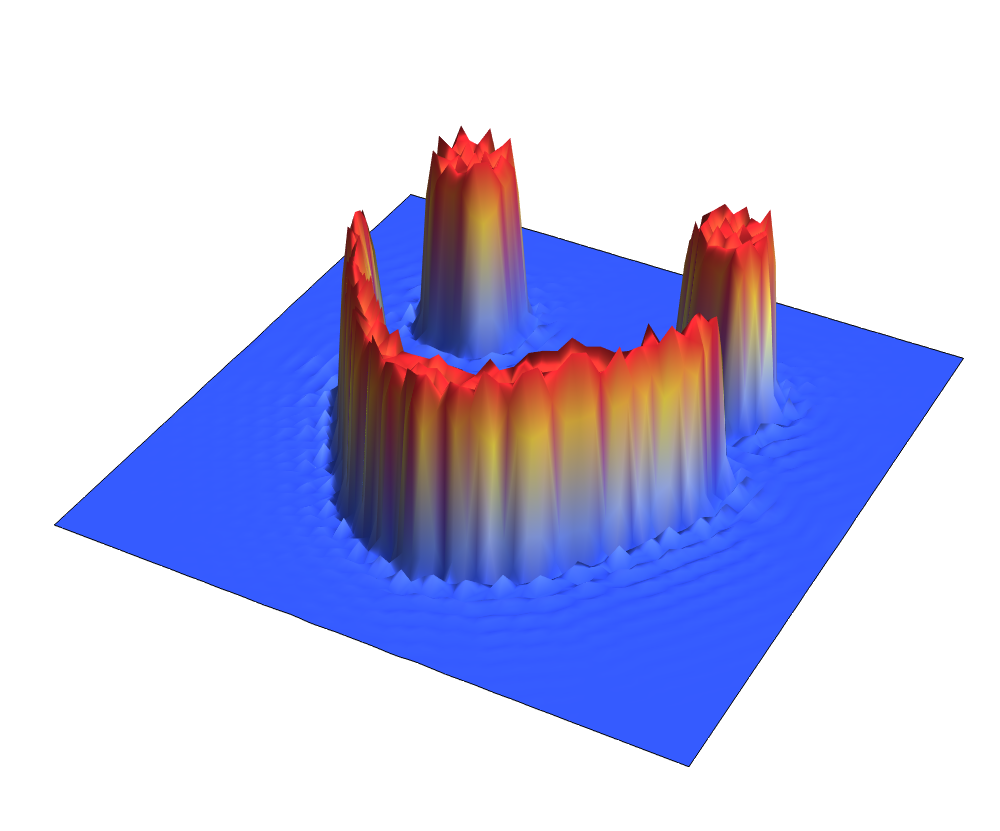}
    \caption{\tiny $t=0.000, \sigma^2 =20$}
  \end{subfigure}
  \hfill
  \begin{subfigure}[b]{0.18\textwidth}
    \centering
    \includegraphics[width=\textwidth]{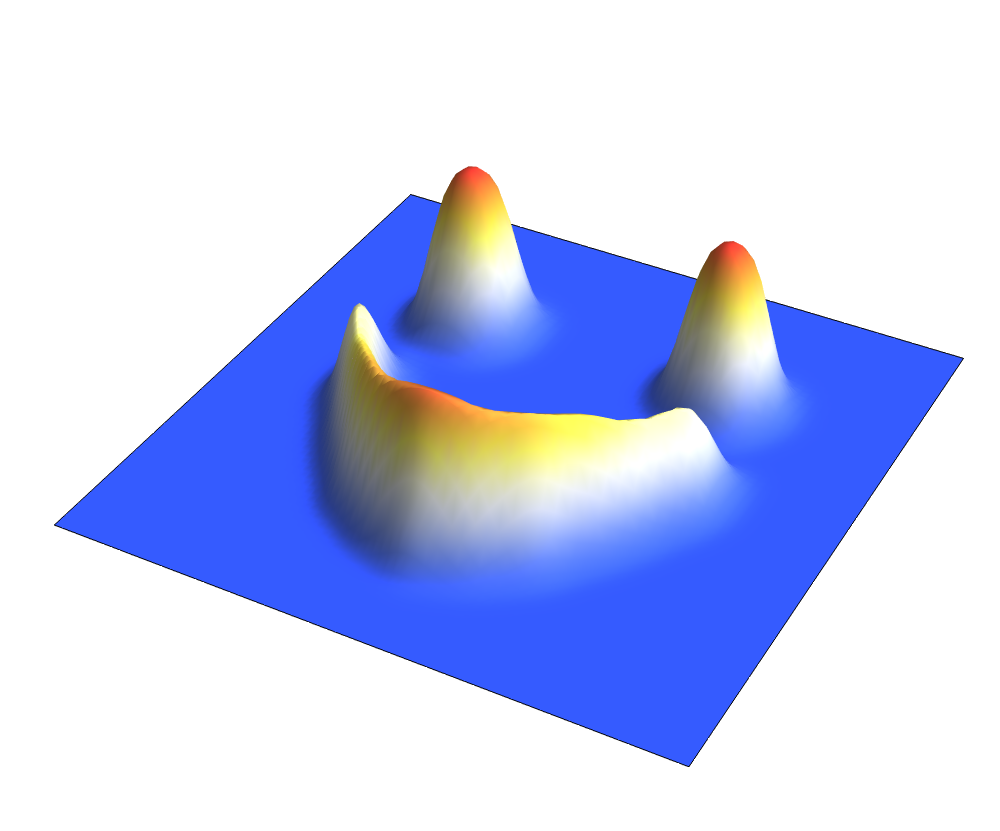}
    \caption{\tiny $t=0.015, \sigma^2 =20$}
  \end{subfigure}
  \hfill
  \begin{subfigure}[b]{0.18\textwidth}
    \centering
    \includegraphics[width=\textwidth]{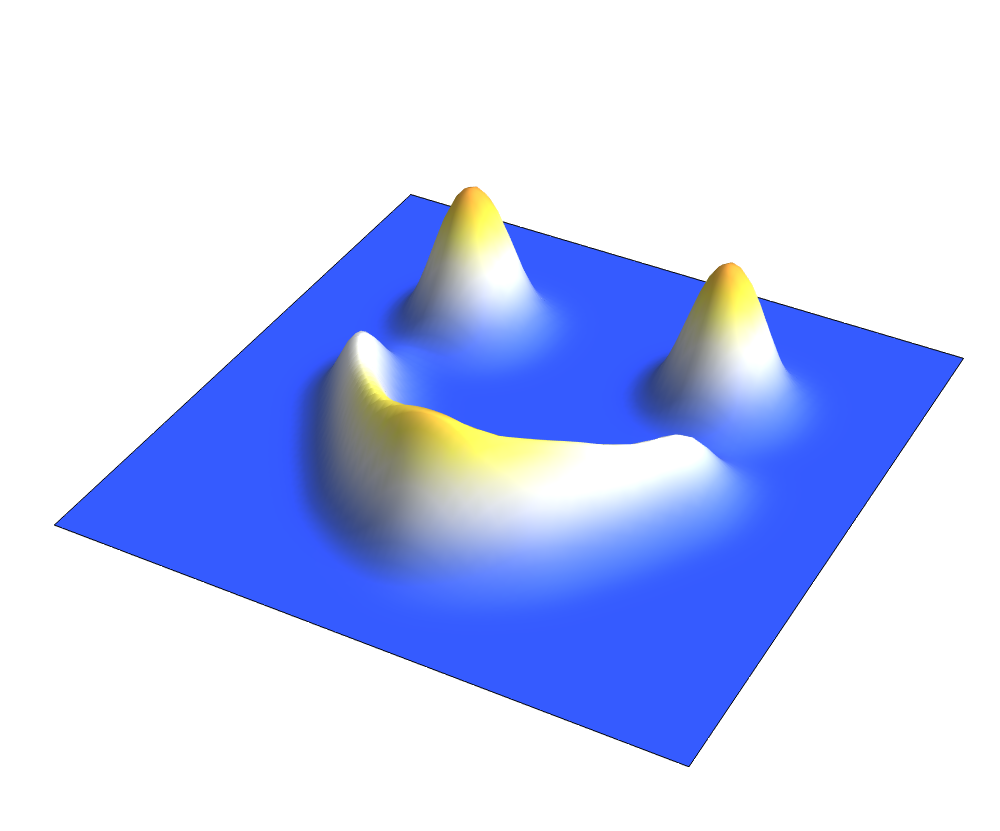}
    \caption{\tiny $t=0.03,\sigma^2 =20$}
  \end{subfigure}
  \hfill
  \begin{subfigure}[b]{0.18\textwidth}
    \centering
    \includegraphics[width=\textwidth]{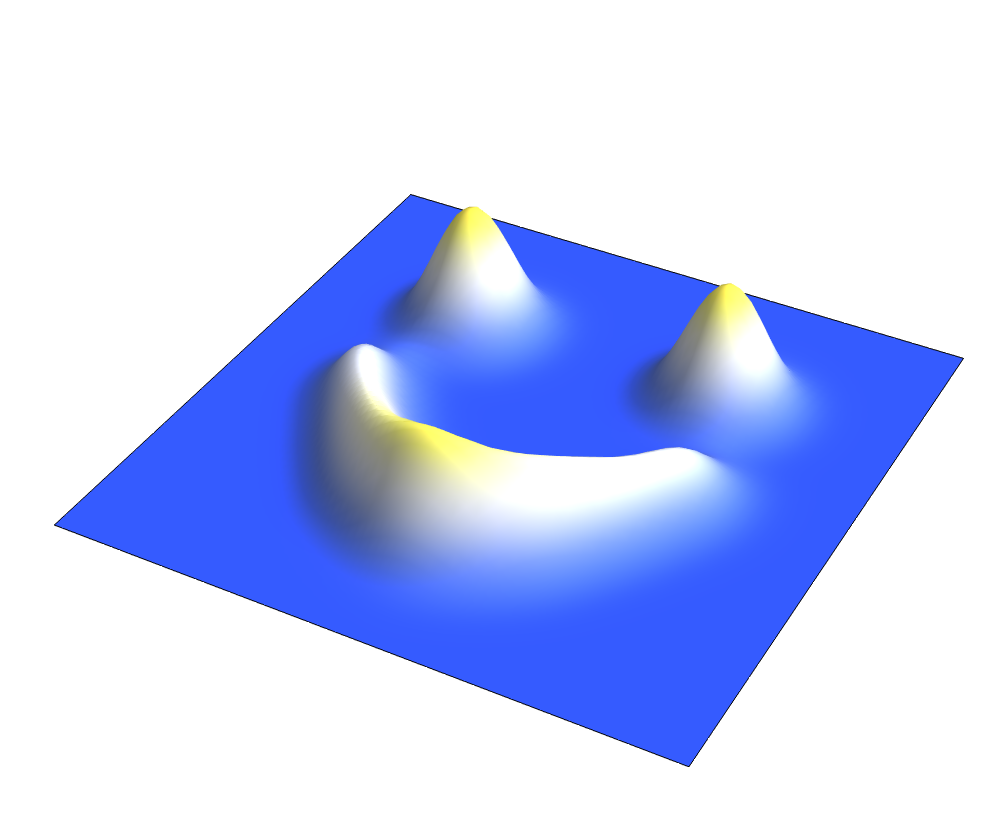}
    \caption{\tiny $t=0.045, \sigma^2 =20$}
  \end{subfigure}
  \hfill
  \begin{subfigure}[b]{0.18\textwidth}
    \centering
    \includegraphics[width=\textwidth]{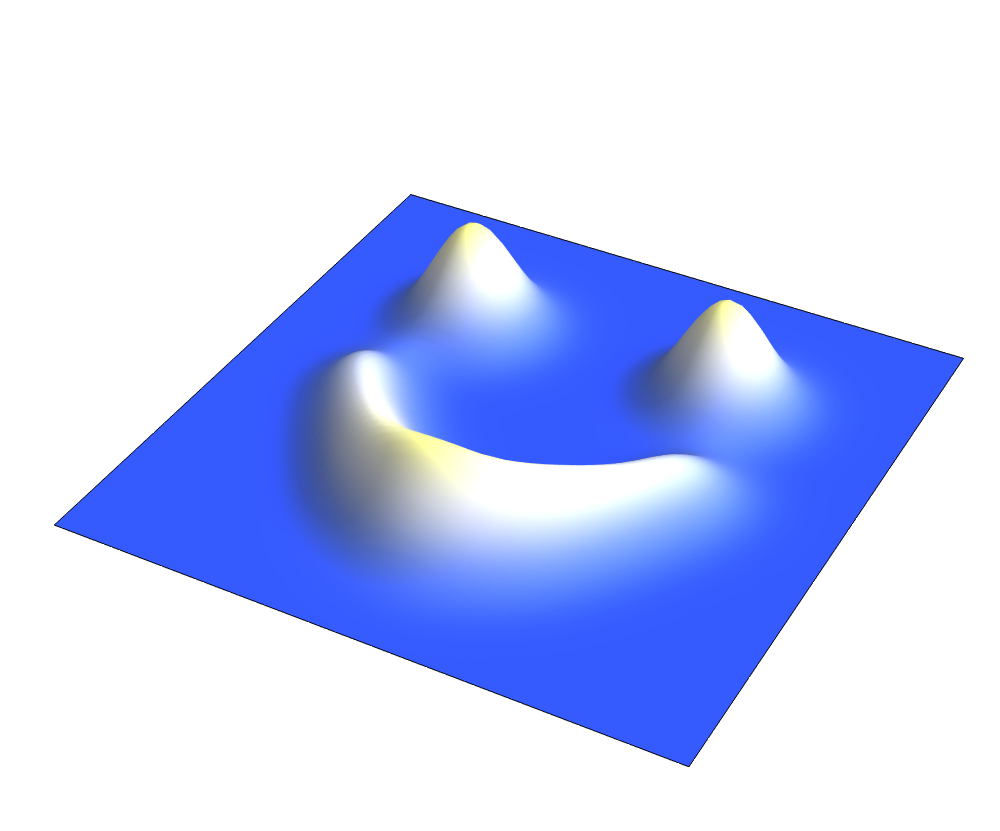}
    \caption{\tiny $t=0.06, \sigma^2 =20$}
  \end{subfigure}
  \caption{Heat dissipation in 2D at 5 timepoints, with scale parameters $\sigma^2 = 2$ and $20$.
  The parameter $\sigma^2$ in the Gaussian measure regulates how strongly the learned function follows the initial data.
  Animations can be found in the ancillary files, under the filenames \texttt{2Dheat\_2.mp4} and \texttt{2Dheat\_20.mp4}.}
  \label{fig:face_frames}
  \vskip -0.2in
\end{figure*}

GPs are a classical tool for purely data based simulation models.
Hence, they appear regularly with their standard covariance functions as an approximate model inside models connected to differential equations
\cite{chai2008multi, zhao2011pde, bilionis2013multi, klenske2015gaussian, ulaganathan2016performance, rai2019gaussian, chen2021solving}.
Furthermore, a huge class of probabilistic ODE solvers \cite{calderhead2009accelerating,schober2014probabilistic,marco2015automatic, schober2019probabilistic,kramer2021linear,tronarp2021bayesian,bosch2021calibrated,schmidt2021probabilistic} and a smaller class of probabilistic PDE solvers \cite{bilionis2016probabilistic, cockayne2017probabilistic, kramer2022probabilistic} make use of GPs when dealing with non-linear differential equations, without constructing new covariance functions.
For systems of PDEs, solutions can be propagated forward in time using numerical discretization combined with GPs \cite{raissi2018numerical}.

Differential equations are often used together with boundary conditions.
There is recent interest in constructing GP priors encoding such boundary conditions \cite{tan2018gaussian,solin2019know,gulian2020gaussian,nicholson2022kernel} and even work constructing GP priors combining differential equations with boundary conditions \cite{LH_AlgorithmicLinearlyConstrainedGaussianProcessesBoundaryConditions,langehegermann2022boundary}. %

\section{Examples}\label{sec:examples}

We demonstrate (S-)EPGPs on three systems of PDEs and a fourth one in \Cref{app:code}.
While the systems presented below are simple, they are all fundamental physical systems still subject to active research, in particular in the field of finite element methods \cite{steinbach2019stabilized,gopalakrishnan2017mapped,perugia2020tent}.
Due to their algebraic simplicity, we omit here details regarding the computation of their corresponding Noetherian multipliers and characteristic varieties.

We compare our method with a version of PINN \cite{raissi2019physics}, implementing some of the recent improvements in the review paper by \citet{cuomo2022scientific}.

We note that there were no previously known GP priors for any of these systems, as none of them are controllable.

The code used to generate figures and tables is available at
\begin{center}
  \begin{small}
    \url{https://github.com/haerski/EPGP}
  \end{small}
\end{center}

Animated versions of some of the figures can be found on the expository website
\begin{center}
  \begin{small}
    \url{https://mathrepo.mis.mpg.de/EPGP/}
  \end{small}
\end{center}
Copies of the codebase and animations can also be found in the ancillary files.

\subsection{Heat equation}\label{ssec:heat_eqn}

\begin{figure*}[b]
  \vskip 0.2in
  \begin{center}
    \includegraphics[width=\linewidth]{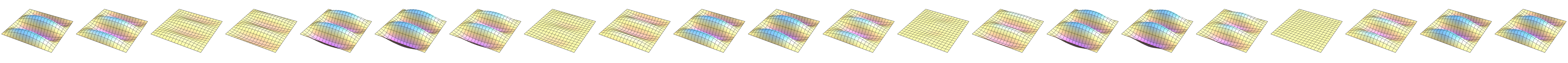}
    \includegraphics[width=\linewidth]{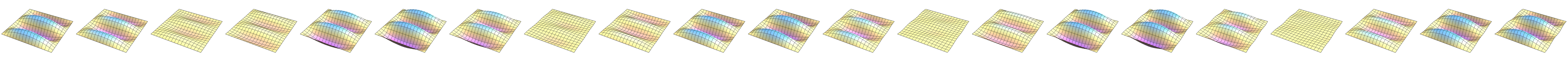}
    \includegraphics[width=\linewidth]{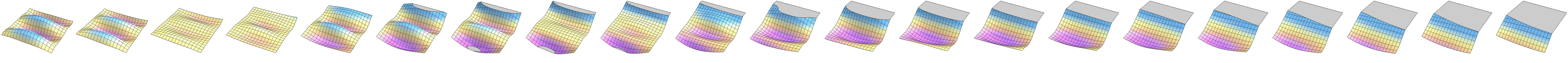}
    \caption{Solutions to wave equations for $t \in [0, 1]$. 
      The top row shows a numerical solution, the first three frames of which serve as the training data.
      The second row is the mean from EPGP (ours) and the third row is the solution from by PINN.
    The gray regions in the PINN solutions are values $>0.3$.
    The ancillary files \texttt{wave\_numerical.mp4}, \texttt{wave\_s-epgp.mp4}, and \texttt{wave\_pinn.mp4} contain animated versions of the above frame captures.}
    \label{fig:waves}
  \end{center}
  \vskip -0.2in
\end{figure*}

The one-dimensional heat equation is given by the PDE $\partial_x^2 u(x,t) = \partial_t u(x,t)$.
Our first goal is to infer an exact solution
purely from sampled data points, without any knowledge about boundary conditions.
Consider the domain $(x,t) \in [-5,5] \times [0,5]$ on a $101 \times 51$ grid of equally spaced points.
These 5151 corresponding function values serve as our ``underlying truth''.

We compare our method with PINN \cite{raissi2019physics} in two setups.
First, we test the ability of the model to solve the heat equation given initial data.
Therefore, we train on different numbers of randomly chosen points at $t = 0$ and study the mean square error over all time points $t \in [0,5]$.
In our second setup, we test the  ability of the model to interpolate the underlying true solution from a limited set of data scattered throughout time.
We train on different numbers of points chosen uniformly at random over the $101 \times 51$ grid.
The results are depicted in \Cref{fig:heat1D}.
The GP achieves an error several orders of magnitude smaller than the errors of PINN, even with fewer data points.
In addition, there is a drastic difference in total computation time between EPGP (10s) and PINN (2h) using an Nvidia A100 GPU.

Next, we apply the EPGP to the 2D heat equation, with an added scale parameter on the Gaussian measure as discussed in the end of \Cref{sec:epgp}.
The initial data is given at $t=0$, on a $101 \times 101$ grid in the square $[-5,5]^2$, where every value is equal to 0, except for a region depicting a smiling face where we set the value to 1.
In this case, the scaling factor $\sigma^2$ in the covariance kernel determines how strongly the initial data is respected.
\Cref{fig:face_frames} compares the posterior mean at fives timepoints.
When $\sigma^2 = 20$, the prior allows abrupt changes and the inferred function conforms to the jagged edges in the data.
In contrast, for $\sigma^2 = 2$ the prior prefers smooth interpretations of the initial data.
In both cases we show the instantaneous smoothing behavior at times $t>0$, which is characteristic to solutions of the heat equation.
For details (covariance functions, experimental setup, etc.) and additional comparisons about the heat equation see Appendix~\ref{appendix_heat}.

\subsection{2D wave equation}
Consider the 2D wave equation, given by $\frac{\partial^2 z}{\partial x^2} + \frac{\partial^2 z}{\partial y^2} = \frac{\partial^2 z}{\partial t^2}$.
The solution we are trying to learn is obtained by solving the wave equation numerically, subject to boundary conditions
  $z(0,y,t) = z(1,y,t) = z(x,0,t) = z(x,1,t) = 0$,
and initial conditions
  $z(x,y,0) = \sin(4\pi x) y (1-y)$,  and $\frac{\partial z}{\partial t}(x,y,0) = 0$.
A plot of the numerical solution can be found on the top row of \Cref{fig:waves}.
The recent theoretical papers \cite{henderson2023characterization,henderson2023wave} construct and study a covariance function for the 3D wave equation with initial conditions at $t=0$.

To learn the numerical solution, we split the domain $(x,y,t) \in [0,1]^3$ into a $21 \times 21 \times 21$ grid and use the data at $t = 0, 0.05, 0.1$ for training.
For S-EPGP, we use a sum of 16 Dirac delta kernels, whose positions we learn.
A PINN model, with 15 hidden layers of size 200, was also trained on the same data, but  failed to get adequate extrapolation performance. The bottom row of \Cref{fig:waves} contains a PINN instance trained for 200,000 epochs.
Technical details about wave equation, our experimental setup, and an additional comparison of (S-)EPGP models can be found in \Cref{appendix_wave}.
An example animation of colliding 2-dimensional wavefronts can be found in the file \texttt{crashing\_waves.mp4}, provided as an ancillary file.

\begin{table*}[t]
\caption{Root mean square errors learning an exact solution to Maxwell's equations, using different number of datapoints for training. Top: S-EPGP, with a varying number of Dirac delta measures.
Bottom: PINN. Here HLW stands for ``hidden layer width''.
Each experiment is repeated 10 times}
\label{tbl:maxwell_PINN}
\vskip 0.15in
\begin{center}
\begin{small}
\begin{tabular}{llllll}
\toprule
Deltas & 5 datapoints & 10 datapoints & 50 datapoints & 100 datapoints & 1000 datapoints \\
\midrule
$24$    &  $6.08 \pm 0.797$ &   $8.92 \pm 1.87$ &   $1.38 \pm 0.698$ &      $0.981 \pm 0.361$ &       $0.884 \pm 0.347$ \\
$48$    &  $4.31 \pm 0.431$ &   $6.98 \pm 1.37$ &  $0.356 \pm 0.392$ &       $0.11 \pm 0.101$ &     $0.0298 \pm 0.0295$ \\
$96$   &  $4.21 \pm 0.387$ &  $3.81 \pm 0.747$ &  $0.173 \pm 0.169$ &  $0.00521 \pm 0.00221$ &   $0.00192 \pm 0.00203$ \\
$192$   &   $3.9 \pm 0.302$ &  $3.21 \pm 0.706$ &   $1.22 \pm 0.696$ &     $0.027 \pm 0.0291$ &   $0.00239 \pm 0.00155$ \\
$384$   &  $\mathbf{3.45 \pm 0.364}$ &   $\mathbf{2.4 \pm 0.796}$ &  $\mathbf{0.192 \pm 0.193}$ &   $\mathbf{0.00974 \pm 0.0113}$ &  $\mathbf{0.000469 \pm 0.00017}$ \\
\bottomrule

\toprule
HLW & 5 datapoints    & 10 datapoints   & 50 datapoints   & 100 datapoints  & 1000 datapoints \\
\midrule
$50$        &  $4.71 \pm 0.403$ &  $4.09 \pm 0.781$ &  $1.05 \pm 0.304$ &  $0.742 \pm 0.36$ &     $0.1 \pm 0.0415$ \\
$100$   &  $4.63 \pm 0.469$ &  $4.12 \pm 0.783$ &  $1.03 \pm 0.278$ &  $0.693 \pm 0.31$ &  $0.0948 \pm 0.0272$ \\
$200$   &   $\mathbf{4.72 \pm 0.42}$ &   $\mathbf{4.1 \pm 0.789}$ &  $\mathbf{1.06 \pm 0.281}$ &  $\mathbf{0.73 \pm 0.296}$ &  $\mathbf{0.0924 \pm 0.0237}$ \\
\bottomrule
\end{tabular}
\end{small}
\end{center}
\vskip -0.1in
\end{table*}

\subsection{Maxwell's equations}\label{ssec:maxwell}

The homogeneous Maxwell equations in a vacuum are
\begin{align*}
  \nabla \cdot \mathbf{E} &= 0 & \nabla \times \mathbf{E} &= -\frac{\partial \mathbf{B}}{\partial t} \\
  \nabla \cdot \mathbf{B} &= 0 & \qquad\nabla \times \mathbf{B} &= \frac{\partial \mathbf{E}}{\partial t},
\end{align*}
where $\mathbf{E} = (E_x(x,y,z,t), E_y(x,y,z,t), E_z(x,y,z,t))^T$ is the vector field corresponding to the electric field and $\mathbf{B} = (B_x(x,y,z,t), B_y(x,y,z,t), B_z(x,y,z,t))^T$ is the vector field corresponding to the magnetic field.

We run the S-EPGP algorithm using $m = 4,8,16,32$, and $64$ Dirac delta measures for each of the six multipliers.
For comparison, we repeat the experiment with PINN, where we base hyperparameters on \cite{mathews21uncovering} and  report the results in \Cref{tbl:maxwell_PINN}.
For details about Noetherian multipliers, the characteristic variety and implementation of S-EPGP and PINN see \Cref{sec:app_maxwell}.

The S-EPGP method learns the true underlying solution much better than PINN, achieving errors several orders of magnitude smaller even with a relatively small number of Dirac delta measures.
Runtimes for S-EPGP scale well and outperform PINN.
Our fastest S-EPGP model, with 24 Dirac deltas trained on only 5 points, completes 10000 training epochs in about 60 seconds on an Nvidia A100 GPU, whereas the slowest one, with 384 Dirac deltas trained on 1000 points, takes about 70 seconds to complete 10000 epochs.
In comparison, each PINN model took about 200 seconds to complete 10000 epochs on the same GPU.

For an example using EPGP for generating solutions to Maxwell's equations, see the ancillary files \texttt{maxwell\_E.mp4} and \texttt{maxwell\_B.mp4}.

\section{Discussion}

Our method takes a starkly different approach to solving and learning PDEs compared to other physics informed machine learning methods such as PINN.
As is common in applied non-linear algebra \cite{michalek21invitation}, our philosophy is to remain in the exact setting as much as possible.
This is evidenced by the application of exact symbolic algebraic techniques and the Ehrenpreis-Palamodov Fundamental Principle in the construction of our kernels.
Thus we say that (S-)EPGP is physics \emph{constrained}, as all realizations from our GPs are, by construction, exact solutions to the PDE system.
Being constrained to solutions is not necessarily a disadvantage when dealing with noisy data, or even data that does not properly follow the PDEs.
The usual GP approaches account for such scenarios by including additional components to the covariance function.
Our experiments show the exact approach to be the superior and scalable, in both interpolation (learning) and extrapolation (solving) tasks.
Non-exactness in the form of noise is only introduced at the very last step as we formulate the GP, which ---
unlike PINN--- makes the (S-)EPGP training objective statistically well motivated and enables the usage of well-established techniques for sparse, variational, and approximate GPs.
Furthermore, our method removes the hyperparameter required for tuning PINN's multiple loss functions.

Compared to other methods of learning and solving PDEs, our method is also completely data-driven and algorithmic.
We do not for example distinguish the time dimension from other spacial dimensions, as is often done in numerical methods.
Our method also does not require explicit initial and boundary conditions: data points can be given anywhere in the domain and can consist of function values, derivatives, or any combination thereof.
For vector-valued functions, we can also learn on partial representations of the data, for example using just electric field data to learn a solution to Maxwell's equations in order to infer the corresponding magnetic field.
This makes (S-)EPGP extremely flexible and applicable with minimal domain expertise.

\section*{Acknowledgements}

The second author acknowledges the support of the SAIL network, funded by the Ministerium für Kultur und Wissenschaft of the state Nordrhein-Westfalen in Germany.

The authors thank the organizers of the ``Workshop on Differential Algebra'' in June 2022 at MPI MiS, Leipzig Germany, for providing the perfect atmosphere to start this research.

The authors thank the anonymous referees for carefully reading the manuscript and offering insightful comments and suggestions which improved the paper overall.

\bibliography{references}
\bibliographystyle{icml2023}

\newpage
\appendix
\onecolumn

\section{Proof of Lemma \ref{lemma_pushforward_gaussian}}\label{appendix_proof_lemma_pushforward_gaussian}
Our proof hinges on the fact that Gaussian processes and Gaussian measures are essentially the same concepts when defined on reasonable function spaces. The pushforward property is then easy to check for Gaussian measures.

We will work with a space $\F$ of functions, which can be a general locally convex topological vector space (LCS), but for our purposes, we work with $\F=C^\infty(\Omega)$ endowed with the topology induced by the semi norms $p_\alpha(f)=\sup_{x\in\Omega}|\partial^\alpha f(x)|$ (recall here that we consider $\Omega$ to be the closure of an open, convex, bounded set). This topology induces the structure of a \emph{separable} Fr\'echet space on $\F$. We write $\F^*$ for the topological dual of $\F$ and $\delta_x\colon f\in\F\mapsto f(x)$ for the evaluation functionals. It turns out that $\delta_x$ span a dense set in $\F^*$, a property which will be crucial in the sequel. 
\begin{lemma}\label{lem:density_F*}
    Let $\F=C^\infty(\Omega)$. Then the linear span of $\{\delta_x\}_{x\in\Omega}$ is weakly-* dense in $\F^*$.
\end{lemma}
\begin{proof}
    First, we must check that $\delta_x\in\F^*$ for a fixed $x\in\Omega$, which means that $\delta_x$ is linear on $\F$ and continuous with respect to the topology induced by the semi norms. Linearity is obvious. To check continuity, let $f_j\rightarrow f$ with respect to all semi norms, so, in particular, $f_j\rightarrow f$ uniformly in $\Omega$. This implies that $f_j(x)\rightarrow f(x)$, so $\delta_x(f_j)\rightarrow \delta_x(f)$.

    To prove density, we use the Hahn-Banach extension theorem in locally convex spaces -- in our case $\F^* $ equipped with the weakly-* topology. Denote by $\mathcal H$ the weakly-* closure of the linear span of   $\{\delta_x\}_{x\in\Omega}$ and assume for contradiction that there exists $\phi\in \F^*\setminus \mathcal H$. By the extension theorem, there exists $f\in\F$ such that $L(f)=0$ for all $L\in \mathcal H$ and $\phi(f)=1$. This implies, in particular, that $\delta_x(f)=0$, so $f(x)=0$ for all $x\in\Omega$, so that $f=0$. Finally, this implies that $\phi(f)=0$, which is a contradiction. The lemma is proved.
\end{proof}
To discuss Gaussian processes/measures, we need to turn $\F$ into the right measure space. Fortunately, it turns out that ``all reasonable sigma algebras'' coincide:
\begin{lemma}
    Let $\F=C^\infty(\Omega)$. The following sigma algebras on $\F$ coincide:
    \begin{enumerate}
        \item The Borel sigma algebra (generated by all open sets),
        \item The sigma algebra which makes all linear functionals in $\F^*$ measurable,
        \item The sigma algebra which makes all evaluation functionals measurable.
    \end{enumerate}
\end{lemma}
\begin{proof}
    For the equality of the first two sigma algebras see p.1945 in \cite{rajput1972gaussiana}, which only uses the separability of $\F$. The equality of the second and third sigma algebras follows at once from the  Lemma \ref{lem:density_F*}.
\end{proof}
Henceforth, we will simply denote any of the three sigma algebras above by $\Sigma$. We define \textbf{Gaussian measures} on $(\F,\Sigma)$ as measures such that every $L\in\F^*$ is a (one dimensional) Gaussian random variable on $(\F,\Sigma,\gamma)$. We have the following Fourier transform type characterization, see Theorem 2.2.4 in \cite{bogachev1998gaussian},
\begin{align}\label{eq:GM}
\EE[\exp(\sqrt{-1}L)]=\int_{\F}\exp(\sqrt{-1}L(f))d\gamma(f)=\exp\left(\sqrt{-1}A(L)-\tfrac{1}{2}V(L)\right)\quad\text{for }L\in\F{^*},
\end{align}
where $A$ and $V$ stand for the \emph{average} and \emph{variance} of $L$, which are
$$
A(L)=\EE[L]=\int_{\F}L(f) d\gamma(f),\quad V(L)=C(L,L),\text{ where }C(L,M)=\int_{\F}(L(f)-A(L))(M(f)-A(M))d\gamma(f).
$$
The linear map $A$ is the \emph{mean} of $\gamma$ and  the bilinear form $C$ is its \emph{covariance}.

With these in mind, it is very easy to prove the assertion of Lemma \ref{lemma_pushforward_gaussian} in the case of Gaussian measures:
\begin{lemma}\label{lem:PF-GM}
    Let $\gamma$ be a Gaussian measure on $(\F^\ell,\Sigma)$ and $B\colon \F^{\ell}\rightarrow\F^{\ell''}$ be a linear continuous map. Then $B_*\gamma$ is a Gaussian measure on $(\F^{\ell''},\Sigma)$ with {mean $L\mapsto A(L\circ B)$ and covariance $(L,M)\mapsto C(L\circ B,M\circ B)$} for $L,\,M\in(\F^{\ell''})^*$.

    Here we recall the definition of the pushforward $B_*\gamma$, i.e.\ the measure on $\F^{\ell''}$ such that
    $$
    \int_{\F^{\ell''}}F(f)dB_*\gamma(f)=\int_{\F^\ell }F(Bh)d\gamma(h)
    $$
    for all measurable $F\colon\F^{\ell}\rightarrow[0,\infty)$ such that the right hand side is finite. Equivalently, $B_*\gamma(S)=\gamma(B^{-1}(S))$ for all Borel sets $S\subset\F{^{\ell''}}$.
\end{lemma}
The statement clearly holds for linear continuous maps between locally convex spaces.
\begin{proof}
    We examine for $L\in (\F^{\ell''})^*$
    \begin{align*}
        \EE_{B_*\gamma}[\sqrt{-1}L]&=\int_{\F^{\ell''}}\exp(\sqrt{-1}L(f))dB_*\gamma(f)=\int_{\F^{\ell}}\exp(\sqrt{-1} L(B h))d\gamma(h)=\EE_{\gamma}[\sqrt{-1}L\circ B]\\
        &=\exp\left(\sqrt{-1}A(L\circ B)-\tfrac{1}{2}V(L\circ B)\right),
    \end{align*}
    so $B_*\gamma$ is a Gaussian measure as well. The relations between the means and covariances are obvious.
\end{proof}
We also make the definition of {Gaussian processes} that we are working with precise. Following \cite{rajput1972gaussiana}, we say that $g$ is a \textbf{Gaussian process} on $(\F,\Sigma,\PP)$ if  for all $x\in\Omega$, $g(x)\in\F^*$ is a (one dimensional) random variable such that $(x,f)\in\Omega\times \F\mapsto g(x)(f)\in \RR$ is measurable and for all $m$ and $x_1,\ldots,x_m\in\Omega$, we have that $X=[g(x_1),\ldots, g(x_m)]$ is a jointly Gaussian distribution (to be precise, an $m$-dimensional multivariate Gaussian on $\F$). In terms of characteristic functions, this is to say that 
\begin{align}\label{eq:GP}
\EE[\exp(\sqrt{-1}X\cdot \nu)]=\int_{\F}\exp (\sqrt{-1}X(f)\cdot \nu)d\PP(f)=\exp\left(\sqrt{-1}a\cdot \nu-\tfrac{1}{2}\nu\cdot v\nu\right)\quad\text{for }\nu\in\RR{^m},
\end{align}
where $a=({\mu}(x_i))_{i=1\ldots m}$ and $v=(k(x_i,x_j))_{i,j=1\ldots m}$.
Here $\mu(x)=\EE[g(x)]$ and $k(x,x')=\EE[(g(x)-\mu(x))(g(x')-\mu(x'))]$ are the \emph{mean} and \emph{covariance kernel} of the GP $g$.
We will use the abbreviation $g(x,f)=g(x)(f)$ for $x\in\Omega$, $f\in\F$.

Heuristically, we notice that a Gaussian process as in \eqref{eq:GP} is nothing else than a restriction of a Gaussian measure as in \eqref{eq:GM} for the specific choice of
$
L=\sum_{i=1}^m\nu_i\delta_{x_i}\in\F^*
$ if $\PP=\gamma$.
To say it bluntly, Gaussian measures can be evaluated by all $L\in\F^*$, whereas Gaussian processes can only be evaluated by the span of pointwise evaluations $\delta_x$.
For our choice of $\F$ as the set of smooth functions, Lemma \ref{lem:density_F*} makes it clear that all of $\F^*$ can be approximated to an arbitrary degree by the pointwise evaluations $\delta_x$.
Hence, under these circumstances, Gaussian Processes on $\F$ ``coincide'' with Gaussian Measures on $\F$.
In fact, if we reduce the set up to the finite dimensional case, i.e.\ $\Omega=\{1,2,\ldots, m\}$, so $\F\simeq \RR^m\simeq \F^*$, then indeed GPs and Gaussian measures both coincide with a multivariate Gaussian random variable on $\RR^m$.

Therefore, the backbone to prove Lemma \ref{lemma_pushforward_gaussian} is the following correspondence between GPs and Gaussian measures:
\begin{theorem}\label{thm:GP-GM}
    Let $g$ be a GP on $(\F,\Sigma,\PP)$. Then there exists a Gaussian measure $\gamma_g$ on $(\F,\Sigma)$ such that
    $$
    \gamma_g(S)=\PP(\{f\in\F\colon g(\,\cdot\,,f)\in S\})\quad \text{for }S\in\Sigma.
    $$
    Conversely, let $\gamma$ be a Gaussian measure on $(\F,\Sigma)$. Then there exists a GP $g_\gamma$ on $(\F,\Sigma,\gamma)$ such that ${\gamma}_{g_\gamma}=\gamma$. Here $g_\gamma(x,f)=f(x)$.

    Moreover, if $g$ and $\gamma$ are in a correspondence as above, we have that
$$
\mu(x)=\int_{\F}f(x)d\gamma(f),\quad k(x,x')=\int_{\F}(f(x)-\mu(x))(f(x')-\mu(x'))d\gamma(f)\quad\text{for }x,x'\in\Omega.
$$
    where $\mu$ and $k$ are the mean and covariance functions of $g$. In fact, if $A$ and $C$ are the mean and covariance of $\gamma$, then $\mu(x)=A(\delta_x)$ and $k(x,x')=C(\delta_x,\delta_{x'})$.
\end{theorem}
Similar statements are true for a large variety of spaces, in particular Lebesgue spaces $L^p,\,1\leq p<\infty$ \cite{rajput1972gaussianb} and other locally convex spaces \cite{rajput1972gaussiana,rajput1972gaussianc}. Here we focus on spaces which embed in continuous functions.
\begin{proof}
    The proof of the statement concerning the correspondence between $g$ and $\gamma$ follows immediately from Remark 1 in \cite{rajput1972gaussiana} and Lemma \ref{lem:density_F*}. 

    The statements about the average and covariance are proved as follows: First
    $$
\mu(x)=\EE[g(x)]=\int_{\F}g(x,f)d\PP(f)=\int_{\F}f(x)d\gamma(f)=\int_{\F}\delta_x(f)d\gamma(f)=A(\delta_x).
    $$
    Then
    \begin{align*}
    k(x,x')&=\EE[(g(x)-\mu(x))(g(x')-\mu(x'))]=\int_{\F}(g(x,f)-\mu(x))(g(x',f)-\mu(x'))d\PP(f)\\
    &=\int_{\F}(f(x)-\mu(x))(f(x')-\mu(x'))d\gamma(f)=\int_{\F}(\delta_x(f)-A(\delta_x))(\delta_{x'}(f)-A(\delta_{x'}))d\gamma(f)=C(\delta_x,\delta_{x'}).
    \end{align*}
    This completes the proof.
\end{proof}
We can now finish our task. The proof consists of putting together Theorem \ref{thm:GP-GM} and Lemma \ref{lem:PF-GM}.
\begin{proof}[Proof of Lemma \ref{lemma_pushforward_gaussian}]
    Let $\gamma_g$ be the Gaussian measure given by Theorem \ref{thm:GP-GM}. Write $A$ and $C$ for its mean and covariance. We can then record by Lemma \ref{lem:PF-GM} that $B_*\gamma_g$ is a Gaussian measure with appropriate mean and variance. We would like to link $B_*g$ with $B_*\gamma_g$. %
    We claim that $B_*\gamma_g=\gamma_{B_*g}$. We have that for $S\in\Sigma$ 
    $$
    \gamma_{B_*g}(S)=\PP(\{f\in\F{^{\ell''}}\colon B_*g(\,\cdot\,,f)\in S\})=\PP(\{h\in\F{^{\ell}}\colon g(\,\cdot\,,Bh)\in S\}).
    $$
    On the other hand,
    $$
B_*\gamma_g(S)=\int_{\F^{\ell''}}\chi_S(f)dB_*\gamma_g(f)=\int_{\F^{\ell}}\chi_S(Bh)d\gamma_g(h)=\gamma_g(\{h\in\F{^\ell}\colon Bh\in S\}).
    $$
    By definition of $\gamma_g$, these two quantities coincide.

    We thus proved that that $B_*g$ is the GP that corresponds to the Gaussian measure $B_*\gamma_g$, in the sense of Theorem \ref{thm:GP-GM}. This takes care of proving that the process $B_*g$ is indeed a GP. It remains to retrieve its mean and covariance. By the first formula in Theorem \ref{thm:GP-GM}, we have that the mean of $B_*g$ at $x\in\Omega$ equals
    $$
    \int_{\F^{\ell''}}f(x)d B_*\gamma_g(f)=\int_{\F^{\ell''}}\delta_x(f)d B_*\gamma_g(f)=\int_{\F^{\ell}}\delta_x(Bh)d \gamma_g(h)=\delta_x\left(B\int_{\F^{\ell}}hd \gamma_g(h)\right)=\delta_x(B\mu)
    =(B\mu)(x).
    $$
   We proved that the mean of $B_*g$ is $B\mu$. To deal with the covariance, we introduce the notation $f\otimes h\colon (y,y')\mapsto f(y)h(y')^{T}$ for $f,\,h\in\F^{\ell}$. We crucially note that $\delta_{(x,x')}f\otimes h=f(x)h(x')^T$. 
  Suppressing the subscript from $\gamma_g$, we have
    \begin{align*}
        k_{B_*g}(x,x')&=C_{B_*\gamma}(\delta_x,\delta_{x'})=C_\gamma (\delta_x\circ B,\delta_{x'}\circ B)=\int_{\F^{\ell}}\delta_{(x,x')}[ B(f-\mu)]\otimes [ B(f-\mu)]d\gamma(f)\\
        &=\delta_{(x,x')}\int_{\F^{\ell}}[ B(f-\mu)]\otimes [ B(f-\mu)]d\gamma(f)=\delta_{(x,x')}\left(B\int_{\F^{\ell}}(f-\mu)\otimes (f-\mu)d\gamma(f)(B')^T\right)\\
        &= \delta_{(x,x')}[B k_g (B')^T]=[B k_g (B')^T](x,x'),
    \end{align*}
so  $k_{B_*g}=B k_g (B')^T$. The proof is complete.
\end{proof}

Comparing to \citep[Lemma~2.2]{LH_AlgorithmicLinearlyConstrainedGaussianProcessesBoundaryConditions}, this proof does not need a compatibility assumption between probabilities and operators.

\section{Convergence of Ehrenpreis-Palamodov integrals}\label{app:supporting_function}

In general, integrals of the form
\begin{align*}
 \phi(\xx)= \int_V D(\xx,\zz) e^{\langle \xx, \zz \rangle} \, d\mu(\zz), && k(\xx,\xx')= \int_V D(\xx,\zz) D(\xx',\zz)^H e^{\langle \xx, \zz \rangle} \, d\mu(\zz)
\end{align*}
appearing in the Ehrenpreis-Palamodov theorem and EPGP kernels need not converge.
This makes the choice of measure $\mu$ slightly delicate.
We propose three solutions.
\begin{enumerate}
  \item\label{itm:1} If the variety $V \subset \CC^n$ is the affine cone of a projective variety, i.e.\ $x \in V \iff \lambda x \in V$ for all $\lambda \in \CC$, we can restrict the measure to be supported on purely imaginary points, which equates to replacing $V$ by $V \cap \sqrt{-1}\RR^n$.
    This fact guarantees the convergence of the EPGP kernel when the characteristic variety is such an affine cone.
  \item\label{itm:2} In some cases, the integral may only converge on a subset of the domain of the solution.
    In such cases we can restrict the domain of our solution.
    A concrete example is with the heat equation $\partial_x^2 = \partial_t$, whose Ehrenpreis-Palamodov representation is $\int_\RR e^{\sqrt{-1} ax - a^2t} \, d\mu(a)$ after restricting and parametrizing the measure.
    Here, the relatively benign restriction to $t \geq 0$ makes the integrand bounded.
    This observation makes the heat equation EPGP kernel $\ke(x,t;x',t')$ well defined whenever $t,t' > 0$.
  \item\label{itm:3} A more general approach is to translate the exponent using the so-called \emph{supporting function} of a convex, compact set $\Omega \subseteq \RR^n$, defined as
    \begin{align*}
      H_{\Omega}(\mathrm{w}) = \max_{\mathrm{w} \in \Omega} \langle \zz, \yy \rangle, \quad\text{for }\mathrm{w}\in\RR^n.
    \end{align*}
    The integral is then modified to be
    \begin{align*}
      \phi(\xx) = \int_V D(\xx,\zz) e^{\langle \xx,\zz \rangle} e^{-H_\Omega(\operatorname{Re}(\zz))} \, d\mu(\zz).
    \end{align*}
    Of course, such a modification does change the Ehrenpreis-Palamodov measure to another allowed Ehrenpreis-Palamodov measure.
    This modification makes the real part of the exponent negative, which bounds the magnitude of the integrand and makes $\phi(\xx)$ defined whenever $\xx \in \Omega$.
    The same modification can be carried over to EPGP by substituting $\mathbf{\Psi(\xx,\zz')}$ in \cref{eq:epgp} by
    \begin{align*}
      \mathbf{\Psi}(\xx,\zz) = \sum_j \sum_{\zz \in \mathcal{S}_{\zz'}} D_j{(\xx,\zz)} e^{\langle \xx,\zz \rangle} e^{-H_{\Omega}(\operatorname{Re}(\zz))}
    \end{align*}
    
\end{enumerate}

We  prove that we have convergence of the Ehrenpreis-Palamodov integral for two explicit classes of  bounded measures $\mu$
\begin{itemize}
    \item with compact support, i.e.\ there exists a relatively compact subset $K$ of $V$ such that $\mu(V\setminus K)=0$ and $\mu(K)<\infty$ (e.g. S-EPGP with its Dirac delta measures),
    \item when $d\mu(\zz')=e^{-\tfrac{\|\zz'\|^2}{2}}d\mathcal{L}(\zz')$ and $V=\sqrt{-1}\RR^d$ (e.g. EPGP).
\end{itemize}
The proofs for convergence of EPGP kernels are analogous.

In case \ref{itm:1}, since $\zz$ is purely imaginary and $\xx$ is real, we have that $|e^{\langle \xx, \zz \rangle}|=1$ and, if $\mu$ has compact support,
\begin{align*}
 \int_V |D(\xx,\zz) e^{\langle \xx, \zz \rangle}| \, d\mu(\zz)=\int_K |D(\xx,\zz)|| e^{\langle \xx, \zz \rangle}| \, d\mu(\zz)\leq\int_K \sup_{\Omega\times K}|D| \, d\mu(\zz) =\mu(K)\sup_{\Omega\times K}|D|,
\end{align*}
which is finite and proves the required convergence of the integral. In the case of the squared exponential weight, the only difference is to notice that
$$
\int_{V}|D(\xx,\zz)|e^{-\tfrac{\|\zz'\|^2}{2}}d\mathcal{L}(\zz')<\infty,
$$
which follows since the squared exponential is a rapidly decreasing function.

In the case \ref{itm:2}, we need not assume compact support, just $\mu(\RR)<\infty$, to get
\begin{align*}
    \int_\RR |e^{\sqrt{-1} ax - a^2t}| \, d\mu(a)=\int_\RR e^{- a^2t} \, d\mu(a)\leq \mu(\RR),
\end{align*}
since $e^{- a^2t}\leq 1$ for $t\geq 0$. This then directly includes the case when $\mu$ is a squared exponential.

Finally, we examine the integrand in case \ref{itm:3} which equals
$$
|e^{\langle \xx,\zz \rangle} e^{-H_\Omega(\operatorname{Re}(\zz))}| =|e^{\langle \xx,\zz \rangle-H_\Omega(\operatorname{Re}(\zz))}|=e^{\langle \xx,\operatorname{Re}\zz \rangle-H_\Omega(\operatorname{Re}(\zz))}\leq1,
$$
where the last inequality follows from the definition of $H_\Omega$. We then conclude that
$$
\int_V |D(\xx,\zz) e^{\langle \xx,\zz \rangle} e^{-H_\Omega(\operatorname{Re}(\zz))}| \, d\mu(\zz)\leq \int_V |D(\xx,\zz)| \, d\mu(\zz)\leq \mu(K)\sup_{\Omega\times K}|D|,
$$
where the last inequality follows like in the analysis of case \ref{itm:1}. The case when $\mu$ is squared exponential follows similarly.

\section{Details about S-EPGP}\label{appendix_SEPGP}
In this section we derive the posterior distribution and training objective for S-EPGP.
Recall that our latent functions are of the form
\begin{align*}
  f(\xx) = \sum_{j=1}^m \sum_{i=1}^r w_{i,j} \frac{1}{|S_{\zz'_{i,j}}|} \left( \sum_{\zz \in S_{\zz'_{i,j}}} D_j(\xx, \zz) e^{\langle \xx, \zz \rangle} \right) =: \mathbf{w}^T \pphi(\xx),
\end{align*}
where we assume $\mathbf{w} \sim \mathcal{N}(0, \frac{1}{mr}\Sigma)$.

Given $n$ noisy observations $Y = f(X) + \epsilon$, where $\epsilon \sim \mathcal{N}(0, \sigma_0^2 I)$, the predictive distribution of $F_* = f(X_*)$ at new points $X_*$ is given by
\begin{align*}
  p(F_* \mid Y) = \mathcal{N}(\phi_*^H A^{-1} \Phi Y, \sigma_n^2 \phi_*^H A^{-1} \phi_*),
\end{align*}
where $A = mr\sigma_0^2 \Sigma^{-1} + \Phi \Phi^H$, $\Phi$ is the $mr \times n$ matrix with columns $\phi(\xx)$ for all $n$ training points $\xx$, and $\phi_*$ is the $mr \times n_*$ matrix with columns $\phi(\xx_*)$ for all $n_*$ prediction points $\xx_*$.
Similarly, the log marginal likelihood is, up to a constant summand $C$
\begin{align*}
  \log p(Y \mid \xx, \theta) = -\frac{1}{2 \sigma_0^2} (Y^T Y - Y^T \Phi^H A^{-1} \Phi Y) - \frac{n-mr}{2} \log \sigma_0^2 - \frac{1}{2} \log |\Sigma| - \frac{1}{2} \log |A| + C
\end{align*}
Training the model means finding $\zz'_{i,j} \in \CC^d$, $\sigma_0 > 0$, and $\Sigma = \operatorname{diag}(\sigma_1^2,\dotsc,\sigma_{mr}^2) \in \RR_{\succeq0}$ such that the log-marginal likelihood is maximized.

Note that the main bottleneck in the above computation is the inversion of $A$. Since usually $mr \ll n$, writing the training objective in this form is computationally efficient, since the matrix $A$ has only size $mr \times mr$.
Instead of inverting $A$, we compute a Cholesky decomposition, which also yields the determinant $|A|$.
An example implementation is presented in \Cref{app:code}.

\section{Details on the heat equation}\label{appendix_heat}

The variety $V$ corresponding to the heat equation is the parabola, which we will denote $z_x^2 = z_t$.
The only multiplier $D$ is 1.
This can be confirmed using the Macaulay2 command \texttt{solvePDE}
\begin{verbatim}
i1 : needsPackage "NoetherianOperators";
i2 : R = QQ[x,t];
i3 : solvePDE ideal (x^2-t)

              2
o3 = {{ideal(x  - t), {| 1 |}}}
\end{verbatim}

We may choose $z_x$ as the independent variable, which turns \cref{eq:epgp} into the covariance function
\begin{align}
\begin{split}\label{eq:heat_covar}
  \ke(x,t, x', t') &= \int_\RR e^{i x z_x - t z_x^2} e^{-i x' z_x - t' z_x^2} e^{-\frac{z_x^2}{2}} \, d\mathcal{L}(z_x)\\
                 &= \int_\RR e^{i z_x (x-x') - (t+t' + 1/2) z_x^2} e^{-i x' z_x - t' z_x^2} \, d\mathcal{L}(z_x)
\end{split}
\end{align}
The integral converges whenever $t + t' +1/2 > 0$, which is always the case in our domain.
Of course, changing the scale parameters in $e^{-\frac{z_x^2}{2}}$ changes the area of convergence.  
For the 1D heat equation, we approximate the integral using Monte-Carlo samples.

In the example of \Cref{ssec:heat_eqn}, the exact solution we want to learn is the function
\begin{align*}
  u(x,t) = \frac{\sqrt 5 (64 t^3 + 125 (-3 + x) (-1 + x) (2 + x) - 50 t (-2 + x) (13 + 4 x) + 40 t^2 (16 + 5 x))}{e^{x^2/(5 + 4 t)} (5 + 4 t)^{7/2}} 
\end{align*}
Our experimental setup for PINN is modeled after the one in \cite{mathews21uncovering}.
We use 5 hidden layers, each of dimension 100, and a $\tanh$ activation function.
The loss function is the sum of the mean squared error incurred from the training data and the mean of the square of the value of the heat equation sampled at 100 random points on the domain $(x,t) \in [-5,5] \times [0,5]$.
The neural network is trained for 10,000 epochs and parameters are optimized using Adam, with learning rate $10^{-4}$.

For one particular instance, with the difference between the underlying truth see \Cref{fig:heat_instance}.
The benefit of using carefully crafted covariance functions is also clearly visible in \Cref{fig:heat_numpts_totalerr}, which gathers total errors from all initial point setups in a comparison between EPGP and PINN.

\begin{figure}
  \vskip 0.2in
  \begin{center}
    \includegraphics[width=\linewidth]{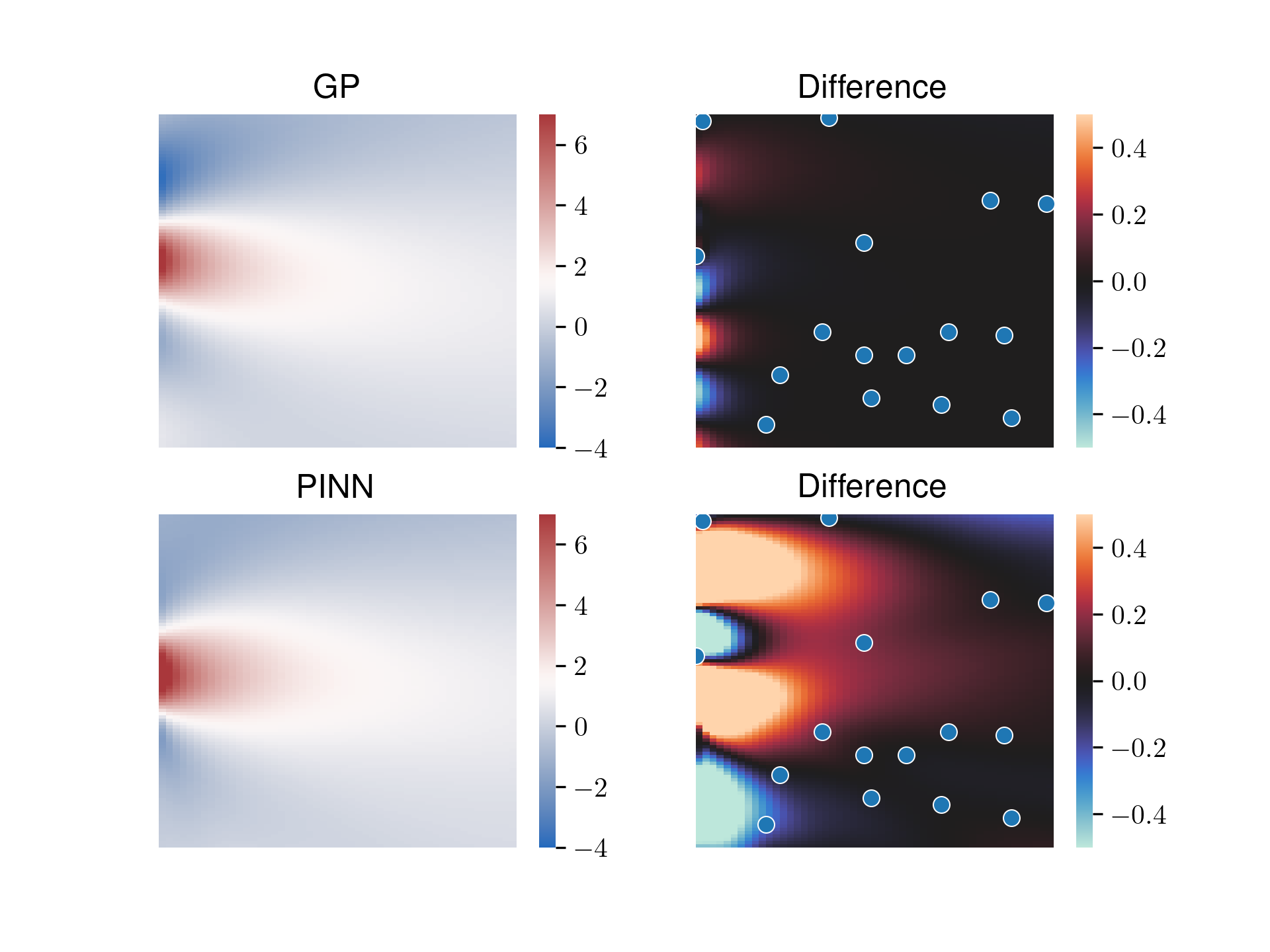}
    \caption{One set of trained instances with 16 randomly chosen datapoints. The vertical axis corresponds to $x$ and the horizontal axis corresponds to $t$, where $t = 0$ on the left and $t = 5$ on the right. The left plots describe the learned heat profile and the right plot denotes the difference between the exact and the learned solution. We observe that PINN performs well in regions where the density of data points is high, such as in the bottom right of the picture, but its error is relatively large at small values of $t$.}
    \label{fig:heat_instance}
\end{center}
\vskip -0.2in
\end{figure}

\begin{figure}
  \vskip 0.2in
  \begin{center}
    \includegraphics[width=\linewidth]{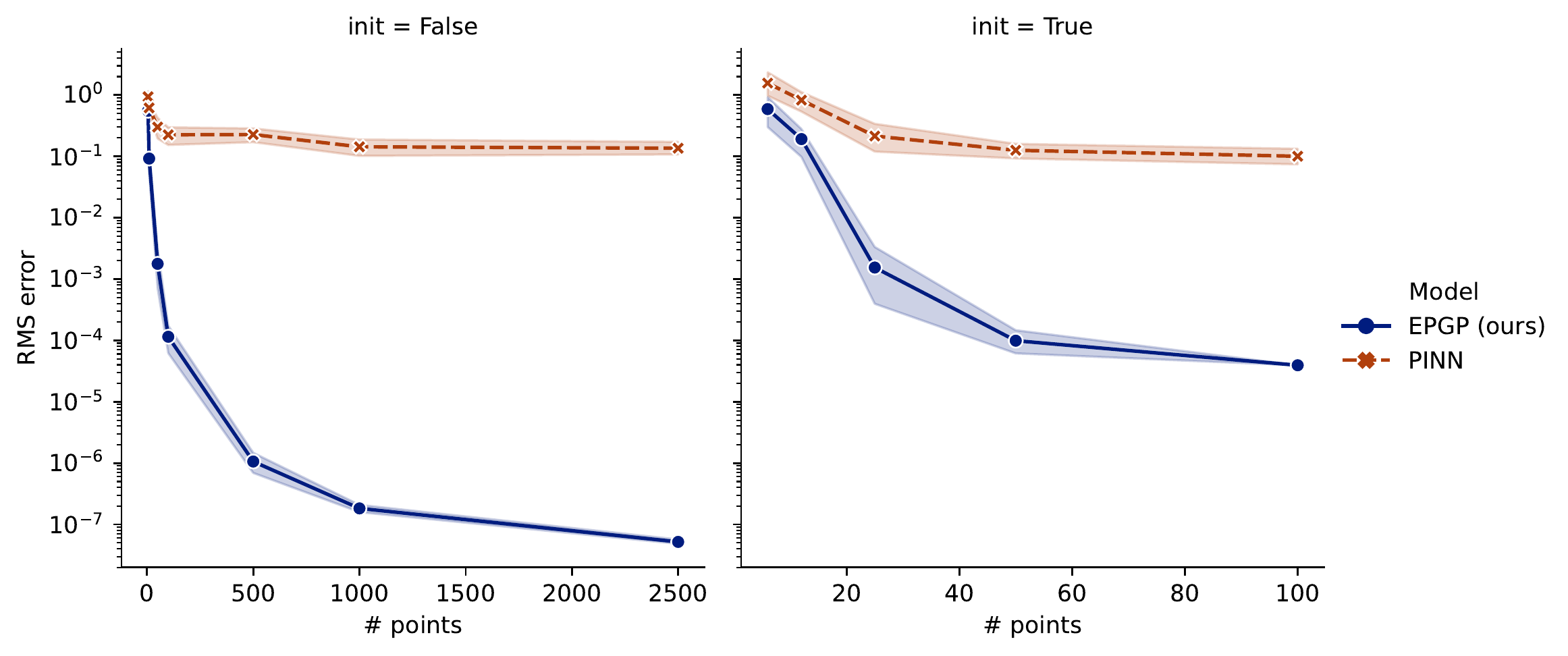}
    \caption{Aggregate plot of total errors with different choices of initial points. %
      On the left (init = False), we choose training points at random in $t \in [0,5]$.
      On the right (init = True), training points are chosen in $t = 0$.
    }
    \label{fig:heat_numpts_totalerr}
  \end{center}
  \vskip -0.2in
\end{figure}

For EPGP, the full $5151 \times 5151$ covariance matrix over 10,000 MC samples takes about one second to compute with an Nvidia A100 GPU.
This is the main computational task in vanilla EPGP, which once completed, allows essentially instantaneous inference by posterior mean.
This means that the entire experimental setup (10 repeats, 12 sets of initial points) takes about 10 seconds in total.
In contrast, each PINN model takes about a minute to complete 10,000 epochs and the computation has to be restarted from scratch for each set of initial points.
The total time used to run the entire experimental setup using PINN is thus about two hours.

For the 2D heat equation, we use the EPGP kernel whose Gaussian measure includes a scale parameter $\sigma^2$.
In this case the EPGP covariance kernel becomes
\begin{align*}
  k_{\sigma}(x,y,t; x', y', t') = \frac{1}{\frac{1}{\sigma^2} + 2(t + t')} e^{-\frac{(x-x')^2 + (y-y')^2}{2(\frac{1}{\sigma^2}+2(t+t'))}}.
\end{align*}
The initial data used in \Cref{fig:face_frames} consists of the values $\{0,1\}$ in the pattern shown in \Cref{fig:smile_initial}.
The initial data is given on a $101 \times 101$ grid of points in the range $(x,y) \in [-5,5]^2$ and at time $t=0$.

\begin{figure}
  \vskip 0.2in
  \centering
  \includegraphics[width=0.3\textwidth]{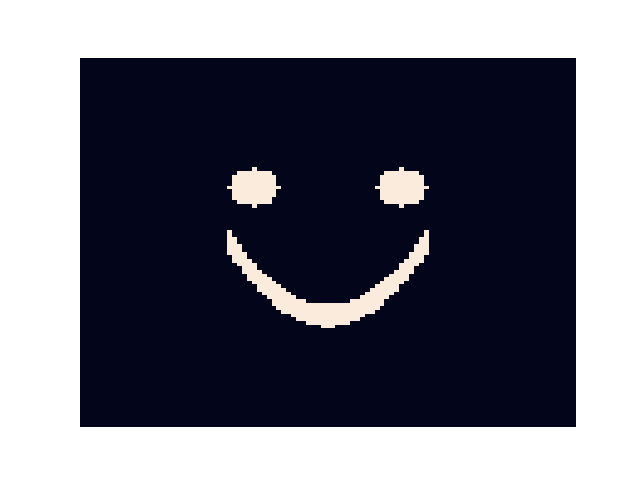}
  \caption{Initial data for learning a solution to the heat equation in 2D. The dark data points correspond to a heat value of 0 and the light points correspond to a heat value of 1.}
  \label{fig:smile_initial}
  \vskip -0.2in
\end{figure}

\section{Details about the wave equation}\label{appendix_wave}
The Macaulay2 command \texttt{solvePDE} reveals that the characteristic variety for the 2D wave equation $\partial_x^2 + \partial_y^2 - \partial_t^2 = 0$ is the cone $x^2 + y^2 - t^2 = 0$, so each entry of $\phi$ in the S-EPGP kernel will have the form
\begin{align*}
  \phi_j(x,y,t) = \frac{1}{2}\left(e^{\sqrt{-1}(xa_j + yb_j + t \sqrt{a_j^2 + b_j^2})} + e^{\sqrt{-1}(xa_j + yb_j - t \sqrt{a_j^2 + b_j^2})}\right)
\end{align*}
where $j = 1,\dotsc,16$. The spectral parameters $a_j, b_j \in \RR$ are learned from the data.

We initialize the 16 pairs $(a_i,b_i) \in \RR^2$ randomly from a standard normal distribution.
The initial noise coefficient $\sigma_0$ is set to $10^{-2}$ and the diagonal matrix $\Sigma$ is initialized to $\frac{1}{16} I$.
Optimization is done using Adam, with a learning rate of $0.1$ for both the $(a_i,b_i)$ and the logarithm of the diagonal entries of $\Sigma$.
Learning rate for the logarithm of $\sigma_0^2$ was set to $0.01$.
The frames visible in the middle row of \Cref{fig:waves} were obtained after 10000 epochs.

For PINN, we again follow a similar schedule to \cite{mathews21uncovering}, with some fine-tuning.
After a few attempts, we settled on 15 hidden layers, each of size 200.
The neural network was trained using the Adam optimizer with learning rate $10^{-4}$ on 200,000 epochs.
In our first attempts, we observed that PINN would converge to a constant solution, which almost certainly is a local optimum of the PINN loss function: a constant surely satisfies the wave equation equation exactly, but does not fit the data very well.
This led us to reweight the PINN objective so that data fit was given a weight 1000 times larger than PDE fit.
Despite our best efforts, we did not manage to get satisfactory extrapolation performance using PINN.

We also ran a comparison of different methods on a solution of the wave equation.
Our underlying true solution $u(x,y,t)$ to the wave equation was computed from the initial values $u(x,y,0) = \exp(-10((x-0.35)^2+(y-0.25)^2))$, $\partial_t u(x,y,0) = 0$, and boundary conditions $u(x,-1,t) = u(x,1,t) = u(-1,y,t) = u(1,y,t) = 0$.
We discretize the domain $(x,y,t) \in [-1,1]^2 \times [0,3]$ by using a $64 \times 64 \times 31$ grid.
Our data consists of triples $(x,y,t)$ chosen uniformly at random on the grid and the corresponding function values $u(x,y,t)$.
We compare the performance of five models: two flavors of S-EPGP, two flavors of (Monte-Carlo approximated) EPGP, and PINN.
For the wave equation, the S-EPGP kernel takes the forms $\ks(x,y,t;x',y',t') = \frac{1}{r} \phi(x,y,t)^H \Sigma \phi(x',y',t')$, where $\phi$ is an $r$-vector with entries
\begin{align*}
  \phi_i(x,y,t) = \frac{1}{2} \left( e^{a_ix + b_iy + \sqrt{a_i^2 + b_i^2} t} + e^{a_ix + b_iy - \sqrt{a_i^2 + b_i^2} t}\right),
\end{align*}
and $\Sigma$ is a diagonal matrix with positive entries.
Our five models are described as follows
\begin{description}
  \item[Complex S-EPGP] This is the implementation of S-EPGP as described in \Cref{ssec:S-EPGP}, where $a_i, b_i$ can take any complex values.
    This model corresponds to rows labeled ``$\CC$ S-EPGP ($r$)'' in \Cref{tbl:wave_comparison}, where $r$ refers to the number of Dirac delta measures used.
  \item[Imaginary S-EPGP] This model restricts $a_i, b_i$ to be purely imaginary numbers, i.e.\ in $\sqrt{-1} \RR$.
    This choice is motivated by the discussion in \Cref{ex:wave}.
    This model corresponds to rows labeled ``$i$ S-EPGP ($r$)'' in \Cref{tbl:wave_comparison}, where $r$ refers to the number of Dirac delta measures used.
  \item[Vanilla EPGP] This is the EPGP model, using a Gaussian measure with variance 2, i.e.\ proportional to $e^{-\frac{a^2+b^2}{2 \cdot 2}}$.
    Since the integral defining the EPGP kernel does not have a closed form solution, we can use a Monte-Carlo approximation.
    This can be implemented with a slight modification of S-EPGP: if $r$ is number of Monte-Carlo points, we can sample $2r$ real values $(a_i, b_i)_{i=1}^r$ randomly from a normal distribution $\mathcal{N}(0,2)$ and set $\Sigma$ to the identity matrix.
    To guarantee convergence, we substitute $(a_i, b_i)$ by $\sqrt{-1}(a_i, b_i)$.
    When learning, we disable optimization of $(a_i, b_i)$ and $\Sigma$ and only optimize for $\sigma_0^2$.
    This model corresponds to rows labeled ``EPGP ($r$)'' in \Cref{tbl:wave_comparison}, where $r$ refers to the number of Monte-Carlo points used.
  \item[Length-Scale EPGP] In this EPGP model, we parametrize our Gaussian weight with a variance parameter.
    The underlying measure is thus proportional to $e^{-\frac{a^2+b^2}{2 \cdot \ell^2}}$.
    Here too we implement the same Monte-Carlo and learning scheme as above, with the addition of the optimization parameter $\ell^2 > 0$.
    This model corresponds to rows labeled ``$\ell^2$ EPGP ($r$)'' in \Cref{tbl:wave_comparison}, where $r$ refers to the number of Monte-Carlo points used.
  \item[PINN] We include PINN mostly for completeness, as despite our best efforts we were unable to get it to converge to anything reasonable.
    We use a neural network with the $\tanh$ activation function.
    At each epoch we use 500 random collocation points in the range $[-1,1]^2 \times [0,3]$ for measuring PDE fit and weigh the data fit summand by a factor of 1000 to avoid converging to a constant solution.
    This model corresponds to rows labeled ``PINN ($h,w$)'' in \Cref{tbl:wave_comparison}, where $h$ and $w$ are respectively the number and width of hidden layers.
\end{description}

For the (S-)EPGP models, we initialize the points $(a_i,b_i)$ from a normal distribution with covariance $2I$.
The matrix $\Sigma$ is initialized to the identity matrix and $\sigma_0^2$ is initially set to $10^{-4}$.
The Adam optimizer is used, with the following learning rates: $0.1$ for $(a_i,b_i)$ (S-EPGP); $0.001$ for $\sigma_0^2$ and the diagonal entries of $\Sigma$ (S-EPGP only); $0.01$ for $\ell^2$ (Length-Scale EPGP only).
The learning rates decay to 0 following a cosine annealing scheduler with warm restarts every 500 epochs \cite{loshchilov2016sgdr}.
Each (S-)EPGP model is trained for 3000 epochs in total.
The PINN model is optimized for 3000 epochs as well using the Adam optimizer with learning rate $10^{-4}$.

Results of the comparison are recorded in \Cref{tbl:wave_comparison}, along with total runtimes for 3000 epochs in \Cref{tbl:wave_time_comparison}.
With very few training points, performance across all flavors of (S-)EPGP are similar.
As the size of the training dataset increases, the added flexibility obtained by increasing the number of Dirac delta measures (S-EPGP) or Monte-Carlo points (EPGP) becomes apparent.
Furthermore, runtimes scale extremely well with respect to the size of the training set.
On the other hand, the PINN models fail to capture the desired solution to the wave equation, despite a large and diverse training set, and they exhibit much higher runtimes in general.

We also note that all of our (S-)EPGP kernels are compatible with standard methods to approximate Gaussian Processes, such as sparse variational Gaussian Processes \citep{titsias09variational}, which may further improve runtime and performance.

\begin{table*}[h]
  \caption{Comparison of root mean squared errors for different models applied to the 2-dimensional wave equation.
  For each model, we predict a $64 \times 64 \times 31$ grid of function values based on a varying number of training points and compare it to the $64 \times 64 \times 31$ grid of ``true'' values, computed numerically.
  The means and standard deviations of the RMS errors are recorded below, with best values in bold.}
\label{tbl:wave_comparison}
\vskip 0.15in
\begin{center}
\begin{small}
\begin{tabular}{lcccc}
\toprule
{} & \multicolumn{4}{c}{Root mean square (RMS) error} \\
\cmidrule(r){2-5}
Training data points &               32   &               128  &               512  &               2048 \\
\midrule
$\CC$ S-EPGP (32)    &  $\mathbf{0.176 \pm 0.032}$ &  $0.103 \pm 0.021$ &  $0.040 \pm 0.003$ &  $0.036 \pm 0.001$ \\
$\CC$ S-EPGP (64)    &  $\mathbf{0.159 \pm 0.036}$ &  $0.100 \pm 0.019$ &  $0.028 \pm 0.003$ &  $0.017 \pm 0.002$ \\
$\CC$ S-EPGP (128)   &  $\mathbf{0.157 \pm 0.025}$ &  $\mathbf{0.089 \pm 0.013}$ &  $\mathbf{0.016 \pm 0.002}$ &  $\mathbf{0.007 \pm 0.000}$ \\
$i$ S-EPGP (32)      &  $\mathbf{0.158 \pm 0.020}$ &  $\mathbf{0.087 \pm 0.012}$ &  $0.056 \pm 0.006$ &  $0.053 \pm 0.004$ \\
$i$ S-EPGP (64)      &  $\mathbf{0.149 \pm 0.022}$ &  $\mathbf{0.072 \pm 0.008}$ &  $0.036 \pm 0.002$ &  $0.032 \pm 0.002$ \\
$i$ S-EPGP (128)     &  $\mathbf{0.143 \pm 0.016}$ &  $\mathbf{0.068 \pm 0.011}$ &  $\mathbf{0.022 \pm 0.004}$ &  $\mathbf{0.010 \pm 0.002}$ \\
EPGP (100)           &  $\mathbf{0.188 \pm 0.067}$ &  $0.099 \pm 0.004$ &  $0.078 \pm 0.008$ &  $0.070 \pm 0.004$ \\
EPGP (1000)          &  $\mathbf{0.133 \pm 0.021}$ &  $\mathbf{0.082 \pm 0.011}$ &  $0.046 \pm 0.004$ &  $0.040 \pm 0.003$ \\
$\ell^2$ EPGP (100)  &  $0.231 \pm 0.053$ &  $\mathbf{0.084 \pm 0.014}$ &  $0.054 \pm 0.009$ &  $0.057 \pm 0.022$ \\
$\ell^2$ EPGP (1000) &  $0.201 \pm 0.033$ &  $\mathbf{0.075 \pm 0.011}$ &  $\mathbf{0.017 \pm 0.003}$ &  $\mathbf{0.008 \pm 0.003}$ \\
PINN (7,100)         &  $0.207 \pm 0.038$ &  $0.133 \pm 0.018$ &  $0.113 \pm 0.004$ &  $0.107 \pm 0.004$ \\
PINN (15,200)        &  $0.192 \pm 0.017$ &  $0.130 \pm 0.018$ &  $0.109 \pm 0.003$ &  $0.107 \pm 0.004$ \\
\bottomrule
\end{tabular}
\end{small}
\end{center}
\vskip -0.1in
\end{table*}

\begin{table*}[h]
  \caption{Comparison of runtimes (in seconds) for different models applied to the 2-dimensional wave equation. 
    Each model is trained for 3000 epochs in total and we record the mean and standard deviation of 10 repetitions for each model, with best values in bold.}
\label{tbl:wave_time_comparison}
\vskip 0.15in
\begin{center}
\begin{small}
\begin{tabular}{lcccc}
\toprule
{} & \multicolumn{4}{c}{Runtime (s)} \\
\cmidrule(r){2-5}
Training data points &            32   &            128  &            512  &            2048 \\
\midrule
$\CC$ S-EPGP (32)    &   $8.4 \pm 0.0$ &   $8.4 \pm 0.1$ &   $8.4 \pm 0.1$ &   $8.7 \pm 0.0$ \\
$\CC$ S-EPGP (64)    &   $8.2 \pm 0.1$ &   $8.3 \pm 0.1$ &   $8.5 \pm 0.0$ &   $8.5 \pm 0.0$ \\
$\CC$ S-EPGP (128)   &   $8.5 \pm 0.1$ &   $8.7 \pm 0.0$ &   $8.8 \pm 0.0$ &   $8.8 \pm 0.0$ \\
$i$ S-EPGP (32)      &   $8.3 \pm 0.0$ &   $8.3 \pm 0.0$ &   $8.3 \pm 0.0$ &   $8.5 \pm 0.1$ \\
$i$ S-EPGP (64)      &   $8.2 \pm 0.0$ &   $8.2 \pm 0.0$ &   $8.4 \pm 0.0$ &   $8.4 \pm 0.0$ \\
$i$ S-EPGP (128)     &   $8.4 \pm 0.0$ &   $8.6 \pm 0.0$ &   $8.6 \pm 0.0$ &   $8.7 \pm 0.0$ \\
EPGP (100)           &   $\mathbf{5.6 \pm 0.1}$ &   $\mathbf{5.6 \pm 0.1}$ &   $\mathbf{5.7 \pm 0.1}$ &   $\mathbf{5.8 \pm 0.1}$ \\
EPGP (1000)          &  $14.1 \pm 0.1$ &  $14.3 \pm 0.1$ &  $15.0 \pm 0.1$ &  $17.6 \pm 0.0$ \\
$\ell^2$ EPGP (100)  &   $7.9 \pm 0.1$ &   $7.9 \pm 0.0$ &   $8.2 \pm 0.0$ &   $8.2 \pm 0.0$ \\
$\ell^2$ EPGP (1000) &  $16.6 \pm 0.0$ &  $17.0 \pm 0.2$ &  $19.0 \pm 0.0$ &  $27.0 \pm 0.1$ \\
PINN (7,100)         &  $41.4 \pm 0.4$ &  $41.4 \pm 0.3$ &  $41.7 \pm 0.3$ &  $42.3 \pm 0.3$ \\
PINN (15,200)        &  $93.8 \pm 1.0$ &  $94.2 \pm 1.0$ &  $94.5 \pm 1.4$ &  $94.7 \pm 0.9$ \\
\bottomrule
\end{tabular}
\end{small}
\end{center}
\vskip -0.1in
\end{table*}

\section{Details about Maxwell's equation}\label{sec:app_maxwell}

If we set $\mathbf{\psi} = (E_x,E_y,E_z,B_x,B_y,B_z)^T$, Maxwell's equations correspond to the following eight linear equations with constant coefficients:
\begin{align*}
  \begin{bmatrix}
    \partial_x&\partial_y&\partial_z&0&0&0\\
    0&-\partial_z&\partial_y&\partial_t&0&0\\
    \partial_z&0&-\partial_x&0&\partial_t&0\\
    -\partial_y&\partial_x&0&0&0&\partial_t\\
    0&0&0&\partial_x&\partial_y&\partial_z\\
    -\partial_t&0&0&0&-\partial_z&\partial_y\\
    0&-\partial_t&0&\partial_z&0&-\partial_x\\
    0&0&-\partial_t&-\partial_y&\partial_x&0
  \end{bmatrix} \mathbf{\psi} = 0.
\end{align*}

The output of the Macaulay2 command \texttt{solvePDE} returns two Noetherian multipliers and one variety, namely an affine cone of spheres.

\begin{verbatim}
i1 : needsPackage "NoetherianOperators"
i2 : R = QQ[x,y,z,t];
i3 : M = matrix {
       {x,y,z,0,0,0},
       {0,-z,y,t,0,0},
       {z,0,-x,0,t,0},
       {-y,x,0,0,0,t},
       {0,0,0,x,y,z},
       {-t,0,0,0,-z,y},
       {0,-t,0,z,0,-x},
       {0,0,-t,-y,x,0}
     };
i4 : solvePDE transpose M

              2    2    2    2
o4 = {{ideal(x  + y  + z  - t ), {|   -xz  |, |   xy  |}}}
                                  |   -yz  |  | y2-t2 |
                                  | -z2+t2 |  |   yz  |
                                  |   -yt  |  |  -zt  |
                                  |   xt   |  |   0   |
                                  |    0   |  |   xt  |
\end{verbatim}

We note that while the two operators are independent and generate the excess dual space \cite{harkonen22thesis}, they are slightly "unbalanced", in the sense that the last two coordinates alone uniquely determine the two summands in the Ehrenpreis-Palamodov representation of the solution.
Thus any potential noise in the $y$ and $z$ coordinates of the magnetic field will have a stronger effect on the quality of the inference procedure.
We solve this imbalance by considering the kernel of the matrix as a map between free $R/P$ modules, where $R = \CC[x,y,z,t]$ is a polynomial ring and $P = \langle x^2 + y^2 + z^2 - t^2 \rangle$ is the prime ideal corresponding to our characteristic variety.
 Since the generators of the kernel as an $R/P$-module maps to a set of $\operatorname{frac}(R/P)$-vector space generators, this procedure indeed yields a valid set of Noetherian multipliers \cite{harkonen22thesis}.
This computation can also be carried out using Macaulay2.

\begin{verbatim}
i5 : N = coker transpose M;
i6 : P = first associatedPrimes N

            2    2    2    2
o6 = ideal(x  + y  + z  - t )

o6 : Ideal of R

i7 : kernel sub(M, R/P)

o7 = image {1} | xz    -y2-z2 xy    -yt   zt    0      |
           {1} | yz    xy     y2-t2 xt    0     -zt    |
           {1} | z2-t2 xz     yz    0     -xt   yt     |
           {1} | yt    0      -zt   xz    xy    -y2-z2 |
           {1} | -xt   zt     0     yz    y2-t2 xy     |
           {1} | 0     -yt    xt    z2-t2 yz    xz     |
\end{verbatim}

We recognize our two Noetherian multipliers in the columns of the above matrix, as well as four extra operators.
The six columns above will serve as our Noetherian multipliers $D_1,\dotsc,D_6$ in the S-EPGP method.
This yields a slightly overparametrized, but also more balanced set of Noetherian multipliers, as every operator has a single zero in a distinct entry.

In order to avoid excessive subscripts, we will depart from our convention denoting primal (space-time) variables by the symbol $\xx$ and dual (spectral) variables by the symbol $\zz$.
Instead, we will use $x,y,z,t$ for the space-time variables and $a,b,c,d$ for the corresponding spectral variables.
Note that the symbols $\mathtt{x},\mathtt{y},\mathtt{z},\mathtt{t}$ in the above matrix denoted by \texttt{o7} actually correspond to $\partial_x, \partial_y, \partial_z, \partial_t$ and thus will be evaluated at the spectral points $(a,b,c,d)$ on the variety $V = V(a^2+b^2+c^2-d^2)$.

For the implicit parametrization trick, we let $a,b,c$ be free variables and solve for $d = \pm \sqrt{a^2 + b^2 + c^2}$.
Thus, as described in \Cref{ssec:S-EPGP}, the S-EPGP kernel for Maxwell's equations will have the form
\begin{align*}
  k(x,y,z,t;x',y',z',t') = \frac{1}{6m} \Phi(x,y,z,t)^H \Sigma \Phi(x',y',z',t'),
\end{align*}
where $\Phi(x,y,z,t)$ is the $(6m \times 6)$ matrix whose rows, indexed by $i = 1,\dotsc, m$ and $j = 1,\dotsc 6$ are
\begin{gather*}
  \frac{1}{2} D_j(a_{ij}, b_{ij}, c_{ij}, \sqrt{a_{ij}^2+b_{ij}^2+c_{ij}^2})^T e^{\sqrt{-1} (a_{ij}x + b_{ij}y + c_{ij}z + \sqrt{a_{ij}^2 + b_{ij}^2 + c_{ij}^2} t)} +\\
  \frac{1}{2} D_j(a_{ij}, b_{ij}, c_{ij}, -\sqrt{a_{ij}^2+b_{ij}^2+c_{ij}^2})^T e^{\sqrt{-1} (a_{ij}x + b_{ij}y + c_{ij}z - \sqrt{a_{ij}^2 + b_{ij}^2 + c_{ij}^2} t)}
\end{gather*}

Our goal is to infer an exact solution to Maxwell's equations from a set of 5, 10, 50, 100, and 1000 randomly selected datapoints in the range $(x,y,z,t) \in [-1,1]^3 \times [0,2]$.
The exact solution is a superposition of five plane waves.
Each plane wave is constructed by choosing two orthogonal $3$-vectors $\mathbf{E_{0,i}}$ and $\mathbf{k_i}$.
We then set
\begin{align*}
  \mathbf{E_i}(x,y,z,t) &= \operatorname{Re}\left( \mathbf{E_{0,i}} e^{\sqrt{-1} \langle \mathbf{k_i}, (x,y,z) \rangle - \|\mathbf{k_i}\|t} \right) \\
  \mathbf{B_i}(x,y,z,t) &= \frac{\mathbf{k_i}}{\|\mathbf{k_i}\|} \times \mathbf{E_i}(x,y,z,t) \\
  \mathbf{E}(x,y,z,t) &= \sum_{i=1}^5 \mathbf{E_i}(x,y,z,t) \\
  \mathbf{B}(x,y,z,t) &= \sum_{i=1}^5 \mathbf{B_i}(x,y,z,t).
\end{align*}
In our experiments, we choose
\begin{align*}
  \mathbf{E_{0,1}} &= \begin{bmatrix} -2 \\ 0 \\ 1 \end{bmatrix} &
  \mathbf{E_{0,2}} &= \begin{bmatrix} 1 \\ 1 \\ 0 \end{bmatrix} &
  \mathbf{E_{0,3}} &= \begin{bmatrix} 1 \\ -1 \\ -1 \end{bmatrix} &
  \mathbf{E_{0,4}} &= \begin{bmatrix} 3 \\ 2 \\ 1 \end{bmatrix} &
  \mathbf{E_{0,5}} &= \begin{bmatrix} -7 \\ 2 \\ 3 \end{bmatrix} \\
  \mathbf{k_1} &= \begin{bmatrix} 1 \\ 0 \\ 2 \end{bmatrix} &
  \mathbf{k_2} &= \begin{bmatrix} 0 \\ 0 \\ 1 \end{bmatrix} &
  \mathbf{k_3} &= \begin{bmatrix} 0 \\ -1 \\ 1 \end{bmatrix} &
  \mathbf{k_4} &= \begin{bmatrix} -1 \\ 1 \\ 1 \end{bmatrix} &
  \mathbf{k_5} &= \begin{bmatrix} 0 \\ 3 \\ -2 \end{bmatrix}
\end{align*}

The exact function is then sampled on a uniform $11 \times 11 \times 11 \times 11$ grid in the ranges $(x,y,z) \in [-1,1]^3$ and $t \in [0,2]$.

For S-EPGP, we initialize the spectral points using standard normal random values.
Each S-EPGP run is optimized using the Adam optimizer with learning rate 0.01 over 10000 epochs.

For PINN, we use 5 hidden layers of varying sizes, with the $\tanh$ activation function.
PDE fit is measured using 500 collocation points sampled uniformly in the region $[-1,1]^3 \times [0,2]$.
The loss function is defined as the sum of the mean squared error at the data points and the mean square error of the PDE constraints, similarly to the original PINN paper \cite{raissi2019physics}.
We train the model for 9000 epochs using the Adam optimizer with learning rate $10^{-3}$ and finally 1000 epochs using the L-BFGS optimizer.

\section{Affine subspaces}\label{sec:app_affine}
In this section, we consider the EPGP kernel in the special case where the characteristic variety $V$ is an affine subspace, i.e.\ linear spaces and translations thereof.

We first show that our approach generalizes the approach to \emph{parametrizable} systems of PDEs in \cite{LH_AlgorithmicLinearlyConstrainedGaussianProcessesBoundaryConditions}.
The control theory literature calls such systems \emph{controllable} \cite{shankar_notes}.
Parametrizable systems are characterized by several algebraic conditions, but the one we are interested in is the following: controllable systems are precisely the ones where the \emph{only} characteristic variety is $\CC^n$.
The Ehrenpreis-Palamodov fundamental principle thus implies that all solutions are of the form
\begin{align*}
  f(\xx) = \sum_{i} \int_{\CC^n} D_i(\zz) e^{\langle \xx, \zz \rangle} \, d\mu_i(\zz) = \sum_i D_i(\partial_x) \int_{\CC^n} e^{\langle \xx, \zz \rangle} \, d\mu_i(\zz) =: \sum_i D_i(\partial_\xx) \phi_i(\xx).
\end{align*}
We can omit the $\xx$-variables in the polynomials $D_i$, since every variable is independent over $R/(0)$, where the zero ideal $(0)$ is the prime ideal corresponding to the variety $\CC^n$ \cite{manssour21linear,harkonen22thesis}.
Furthermore, any choice of smooth functions $\phi_i(\xx)$ yields a solution.
In other words, the set of solutions to the PDEs $A(\partial_\xx)f=0$ is the image of the matrix $B(\partial_\xx)$, which is the matrix with columns $D_i(\partial_\xx)$.
Thus the EPGP kernel induces the pushforward GP of $B$, where our latent covariance is the squared exponential kernel, precisely as in \Cref{ex:no_pde}.

We now generalize this to general affine subspaces, i.e.\ translated linear spaces.
Suppose $A$ describes a system of linear PDEs whose only characteristic variety is an affine subspace.
Then there is a parametrization of the variety of the form $\zz \mapsto C\zz + b$ for some $n \times d$ constant matrix $C$ of rank $d$ and a constant vector $b$.
By a change of variables, we may choose the Noetherian operators to be functions of $\zz$ only, so by Ehrenpreis-Palamodov the solution set consists of summands of the form
\begin{align*}
  f_i(\xx) &= \int_{\CC^d} D_i(\zz) e^{\langle \xx, C\zz \rangle + \langle \xx, b \rangle} \, d\mu_i(\zz) \\
           &= e^{\langle \xx, b \rangle} D_i(\partial_{C^T\xx}) \int_{\CC^d} e^{\langle C^T\xx, \zz \rangle} \, d\mu_i(\zz) \\
           &= e^{\langle \xx, b \rangle} (D_i(\partial_y) \phi_i(y))_{y \to C^T\xx},
\end{align*}
where $\phi_i(y)$ is an arbitrary, smooth $d$-variate latent function.
By Ehrenpreis-Palamodov, \emph{every} smooth solution arises this way.

If we gather all $D_i$ inside a matrix $B$, the EPGP kernel (up to a scaling factor) becomes $\ke(\xx,\xx') = e^{\langle \xx, b \rangle} \left( B(\partial_y) \gamma(y,y') B^T(\partial_{y'}) \right)_{y \to C^T\xx} e^{\langle \xx', b \rangle}$, where $\gamma(y, y')$ is the $d$-dimensional squared exponential kernel.
We observe that $\ke(\xx,\xx')$ is (up to scaling) the covariance function of $f(\xx) = e^{\langle \xx, b \rangle} \left( B(\partial_y) g(y) \right)_{y \to C^T\xx}$, where $g(y)$ is a vector of independent latent GPs with squared exponential covariance.
Since $f(\xx)$ is the general form of a solution to the PDEs given by $A$ and realizations of GPs with squared exponential covariance functions are dense in the set of smooth functions, we conclude that our method constructs a kernel for which realizations are dense in the set of smooth solutions to $A$.

\section{Example implementation for Laplace's equation}\label{app:code}
In this section, we present an example implementation of S-EPGP in PyTorch.
Other examples, including code generating all figures and tables in this paper can be found in the repository
\begin{center}
  \begin{small}
    \url{https://github.com/haerski/EPGP}.
  \end{small}
\end{center}

Our aim is to learn a numerically computed solution to Laplace's equation \(\partial_x^2 + \partial_y^2 = 0\) from data.
All input cells will be framed.
We note that the code presented below is self contained, aside from dependencies on \texttt{torch} (version 1.13.1), \texttt{py-pde} (version 0.27.1), and \texttt{numpy} (version 1.23.1), and was tested on Python version 3.9.13.

We start by importing the required packages.
\begin{lstlisting}[language=Python,frame=trBL]
import torch
import matplotlib.pyplot as plt
import numpy as np
from pde import CartesianGrid, solve_laplace_equation

torch.set_default_dtype(torch.float64)
torch.manual_seed(13);
\end{lstlisting}

\subsection{A numerical solution}

We compute a numerical solution to the Laplace equation in 2D, given by the equation \(\partial_x^2 + \partial_y^2 = 0\).
We consider this as the underlying ``true'' solution to the PDE from which we draw training data.
This solution is plotted in \Cref{fig:code_true}.

\begin{lstlisting}[language=Python,frame=trBL]
grid = CartesianGrid([[0, 2 * np.pi]] * 2, 64)
bcs = [{"value": "sin(y)"}, {"value": "sin(x)"}]

res = solve_laplace_equation(grid, bcs)
\end{lstlisting}

\begin{figure*}
  \vskip 0.2in
  \centering
  \begin{subfigure}[b]{0.33\textwidth}
    \centering
    \includegraphics[width=\textwidth]{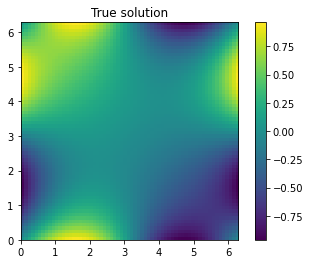}
    \caption{Solution computed numerically.}
    \label{fig:code_true}
  \end{subfigure}
  \hfill
  \begin{subfigure}[b]{0.33\textwidth}
    \centering
    \includegraphics[width=\textwidth]{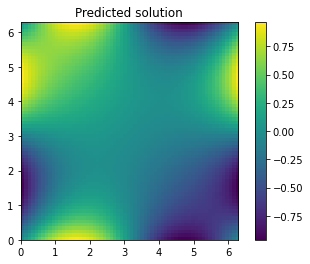}
    \caption{Solution inferred by S-EPGP.}
    \label{fig:code_pred}
  \end{subfigure}
  \hfill
  \begin{subfigure}[b]{0.33\textwidth}
    \centering
    \includegraphics[width=\textwidth]{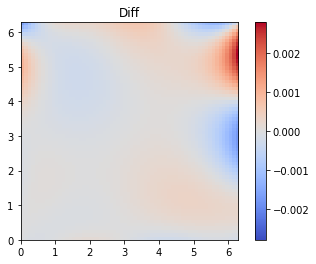}
    \caption{Difference between the solutions.}
    \label{fig:code_diff}
  \end{subfigure}
  \caption{Figures output by code snippets in the example implementation in \Cref{app:code}, depicting solutions to a 2-dimensional Laplace equation, with sinusoidal boundary conditions.}
  \label{fig:code_images}
  \vskip -0.2in
\end{figure*}

We convert the \texttt{py-pde} types to PyTorch tensors

\begin{lstlisting}[language=Python,frame=trBL]
Ps = torch.tensor(grid.cell_coords)
u_true = torch.tensor(res.data)
\end{lstlisting}

The training dataset will consist of 50 randomly sampled points in the numerical
solution

\begin{lstlisting}[language=Python,frame=trBL]
train_pts = 50
train_idx = torch.randperm(len(u_true.flatten()))[:train_pts]
X = Ps.flatten(0,1)[train_idx]
U = u_true.flatten()[train_idx]
\end{lstlisting}

\subsection{Setting up a S-EPGP kernel}

Running the command \texttt{solvePDE} in Macaulay2 reveals two
varieties, namely the lines \( a = i b\) and \(a = - i b\), where
\(a,b\) are spectral variables corresponding to \(x,y\) respectively. For
both lines, there is only one Noetherian multiplier, namely 1. This
means that the Ehrenpreis-Palamodov representation of solutions to
Laplace's equations are of the form \[
\int_{a = ib} e^{ax + by} \, d\mu_1(a,b) + \int_{a=-ib} e^{ax + by} \, d\mu_2(a,b).
\]By parametrizing the two lines, we can rewrite the integrals in a
simpler form. We use the parametrizations
\begin{align*}
  (a,b) &= ((1+i)c, (1-i)c) \\ (a,b) &= ((1-i)c, (1+i)c)
\end{align*}

The integrals then become \[
\int_{\mathbb C} e^{(1+i)cx + (1-i)cy} \, d\mu_1'(c) + \int_{\mathbb C} e^{(1-i)cx + (1+i)cy} \, d\mu_2'(c)
\]We approximate each measure with \(m\) Dirac delta measures. This
translates to the S-EPGP kernel
\[ k(x,y) = \Phi(x,y)^H \Sigma \Phi(x,y), \] where \(\Phi(x,y)\) is the
vector with entries \[ 
\Phi(x,y)_j = \begin{cases} e^{(1+i)c_jx + (1-i)c_jy}, & \text{if } j = 1,\dotsc,m \\ e^{(1-i)c_jx + (1+i)c_jy}, & \text{if } j = m+1, \dotsc, 2m \end{cases},
\] and \(\Sigma\) is a \(2m \times 2m\) diagonal matrix with positive
entries \(\sigma_j^2\). Our goal will be to learn the
\(c_j \in \mathbb C, \sigma_j^2 > 0\) that minimize the log-marginal
likelihood. Given an array \texttt{c} of length \texttt{2m} and a
\texttt{s × 2} matrix \texttt{X} of points \texttt{(x,y)}, the function
\texttt{Phi} returns the \texttt{2m × s} matrix with columns
\(\Phi(x,y)\).

\begin{lstlisting}[language=Python,frame=trBL]
def Phi(c,X):
    c1, c2 = c.chunk(2)
    c1 = c1.unsqueeze(1) * torch.tensor([1+1.j,1-1.j])
    c2 = c2.unsqueeze(1) * torch.tensor([1-1.j,1+1.j])
    cc = torch.cat([c1,c2])
    return cc.inner(X).exp()
\end{lstlisting}

\subsection{Objective function}

Suppose we are trying to learn on \(s\) data points. Let \(X\) be the
\(s \times 2\) matrix with input points \(x_k, y_k\) and \(U\) the
\(s \times 1\) vector with output values \(u_k\). Let \(\Phi\) be the
\(2m \times s\) matrix of features obtained by the function \texttt{Phi}
above.

The negative log-marginal likelihood function is \[
\frac{1}{2\sigma_0^2}(U^TU - U^T\Phi^HA^{-1}\Phi U) + \frac{s - 2m}{2} \log \sigma_0^2 + \frac{1}{2} \log |A| + \frac{1}{2} \log |\Sigma| + \frac{n}{2} \log 2\pi,
\] where \(\sigma_0^2\) is a noise coefficient and
\[ A = \Phi \Phi^H + \sigma_0^2 \Sigma^{-1} \]This can be computed
efficiently using a Cholesky decomposition: \(A = LL^H\). Ignoring
constants and exploiting the structure of \(\Sigma\), we get the
objective function \[
\frac{1}{2\sigma^2}(\|U\|^2 - \|L^{-1}\Phi U\|^2) + \frac{s-2m}{2} \log \sigma_0^2 + \sum_{j=1}^{2m} \log L_{j,j} + \frac{1}{2} \sum_{j=1}^{2m} \log \sigma_j^2
\]

The function below computes the Negative Log-Marginal Likelihood (NLML).
Here we assume that \texttt{Sigma} is a length \(2m\) vector of values
\(\log \sigma_j^2\) and \texttt{sigma0} is \(\log \sigma_0^2\)

\begin{lstlisting}[language=Python,frame=trBL]
def NLML(X,U,c,Sigma,sigma0):
    phi = Phi(c,X)
    A = phi @ phi.H + torch.diag_embed((sigma0-Sigma).exp())
    L = torch.linalg.cholesky(A)
    alpha = torch.linalg.solve_triangular(L, phi @ U, upper=False)
    nlml = 1/(2*sigma0.exp()) * (U.norm().square() - alpha.norm().square())
    nlml += (phi.shape[1] - phi.shape[0])/2 * sigma0
    nlml += L.diag().real.log().sum()
    nlml += 1/2 * Sigma.sum()
    return nlml
\end{lstlisting}

\subsection{Training}

We now set up parameters, initial values, optimizers and the training
routine. We will use \(m = 8\) Dirac delta measures for each integral.

\begin{lstlisting}[language=Python,frame=trBL]
m = 8
Sigma = torch.full((2*m,), -np.log(2*m)).requires_grad_()
sigma0 = torch.tensor(np.log(1e-5)).requires_grad_()
c = (1*torch.randn(2*m, dtype=torch.complex128)).requires_grad_()

U = U.to(torch.complex128).reshape(-1,1)
X = X.to(torch.complex128)
\end{lstlisting}

\begin{lstlisting}[language=Python,frame=trBL]
def train(opt, sched, epoch_max = 1000):
    for epoch in range(epoch_max):
        nlml = NLML(X,U,c,Sigma,sigma0)

        print(f'Epoch {epoch+1}/{epoch_max}\tNLML {nlml.detach():.3f}', end='\r')

        opt.zero_grad()
        nlml.backward()
        opt.step()
        sched.step()
\end{lstlisting}

Here we use a simple Adam optimizer, with learning rate 0.1 and
decaying by a factor of 10 every 1000 steps. We train for 3000 epochs.

\begin{lstlisting}[language=Python,frame=trBL]
opt = torch.optim.Adam([c,Sigma,sigma0], lr = 1e-2)
sched = torch.optim.lr_scheduler.StepLR(opt,3000,gamma=0.1)
train(opt,sched,3000)
\end{lstlisting}

\subsection{Prediction}

Suppose we want to use our trained model to predict the value of the
function at \(r\) points \((x_i,y_i)_{i=1}^r\), organized in the \(r \times 2\)
matrix \(X_*\). We will do inference using the posterior mean, which is
given by \[
\Phi_*^H A^{-1} \Phi U,
\] where \(\Phi_*\) is the \(2m \times r\) matrix, with columns
\(\Phi(x_i,y_i)\) for each row in \(X_*\).

The following function computes the prediction, where the variable
\texttt{X\_} corresponds to \(X_*\).
For numerical stability, we compute a Cholesky decomposition of $A$ instead of inverting.
Since Laplace's equation has real coefficients, the real part of a solution is yet again a solution.
As we are only looking for real valued functions, we will discard the imaginary part of the predicted values.

\begin{lstlisting}[language=Python,frame=trBL]
def predict(X_, X,U,c,Sigma,sigma0):
    with torch.no_grad():
        phi = Phi(c,X)
        A = phi @ phi.H + torch.diag_embed((sigma0-Sigma).exp())
        L = torch.linalg.cholesky(A)

        alpha = torch.linalg.solve_triangular(L, phi @ U, upper=False)
        alpha1 = torch.linalg.solve_triangular(L.H, alpha, upper=True)

        phi_ = Phi(c,X_)
        return (phi_.H @ alpha1).real    
\end{lstlisting}

We compute predicted values on the same points as the numerical solution.
\begin{lstlisting}[language=Python,frame=trBL]
X_ = Ps.flatten(0,1).to(torch.complex128)
u_pred = predict(X_, X, U, c, Sigma, sigma0)
\end{lstlisting}

The root mean square error of our prediction, computed by the code snippet below, is approximately $3.67 \cdot 10^{-4}$.

\begin{lstlisting}[language=Python,frame=trBL]
(u_pred.view_as(u_true) - u_true).square().mean().sqrt().item()
\end{lstlisting}

We can also visually compare the true solution with our prediction.
The following two snippets generate Figures~\ref{fig:code_true} and \ref{fig:code_pred}.

\begin{lstlisting}[language=Python,frame=trBL]
ax = plt.imshow(u_true,extent=2*[0,2*np.pi])
plt.colorbar(ax)
plt.title("True solution");
\end{lstlisting}

\begin{lstlisting}[language=Python,frame=trBL]
ax = plt.imshow(u_pred.view_as(u_true),extent=2*[0,2*np.pi])
plt.colorbar(ax)
plt.title("Predicted solution");
\end{lstlisting}

Finally, we plot the difference between the true and predicted solutions.
The plot is depicted in \Cref{fig:code_diff}.

\begin{lstlisting}[language=Python,frame=trBL]
diff = u_pred.view_as(u_true) - u_true
limit = max(diff.max(), -diff.min())
ax = plt.imshow(diff, extent=2*[0,2*np.pi],
                  cmap='coolwarm', vmin = -limit, vmax = limit)
plt.colorbar(ax)
plt.title("Diff");
\end{lstlisting}

\end{document}